\documentclass[twocolumn]{article}

\usepackage{microtype}
\usepackage{graphicx}
\usepackage{subfigure}
\usepackage{booktabs} 

\usepackage{hyperref}

\usepackage{amsmath}
\usepackage{amssymb}
\usepackage{mathtools}
\usepackage{amsthm}

\usepackage[capitalize,noabbrev]{cleveref}

\theoremstyle{plain}
\newtheorem{theorem}{Theorem}[section]

\newtheorem{lemma}[theorem]{Lemma}

\theoremstyle{definition}
\newtheorem{definition}[theorem]{Definition}

\newtheorem{remark}[theorem]{Remark}

\usepackage[textsize=tiny]{todonotes}

\usepackage{bm}
\newcommand{\argmax}{\mathop{\rm argmax}\limits}
\newcommand{\argmin}{\mathop{\rm argmin}\limits}

\newcommand{\targmin}{{\rm argmin}}

\newcommand{\sqtimes}{{\raisebox{0.08em}{$\times$}\hspace{-0.79em}\Box}}
\newcommand{\lammax}{{\lambda_\mathrm{max}}}

\newcommand*\samethanks[1][\value{footnote}]{\footnotemark[#1]} 
\title{Distributionally Robust Safe Screening}

\author{
Hiroyuki Hanada\thanks{RIKEN, Wako, Saitama, Japan}~\thanks{{\tt hiroyuki.hanada@riken.jp}}
\and Satoshi Akahane\thanks{Nagoya University, Nagoya, Aichi, Japan}
\and Tatsuya Aoyama\samethanks[3]
\and Tomonari Tanaka\samethanks[3]
\and Yoshito Okura\samethanks[3]
\and Yu Inatsu\thanks{Nagoya Institute of Technology, Nagoya, Aichi, Japan}
\and Noriaki Hashimoto\samethanks[1]
\and Taro Murayama\thanks{DENSO CORPORATION, Kariya, Aichi, Japan}
\and Lee Hanju\samethanks[5]
\and Shinya Kojima\samethanks[5]
\and Ichiro Takeuchi\samethanks[3]~\samethanks[1]~\thanks{{\tt ichiro.takeuchi@mae.nagoya-u.ac.jp}}
}
\begin{document}

\maketitle




\begin{abstract}
 In this study, we propose a method \emph{Distributionally Robust Safe Screening (DRSS)}, for identifying unnecessary samples and features within a DR covariate shift setting.
This method effectively combines DR learning, a paradigm aimed at enhancing model robustness against variations in data distribution, with safe screening (SS), a sparse optimization technique designed to identify irrelevant samples and features prior to model training.
The core concept of the DRSS method involves reformulating the DR covariate-shift problem as a weighted empirical risk minimization problem, where the weights are subject to uncertainty within a predetermined range.
By extending the SS technique to accommodate this weight uncertainty, the DRSS method is capable of reliably identifying unnecessary samples and features under any future distribution within a specified range.
We provide a theoretical guarantee of the DRSS method and validate its performance through numerical experiments on both synthetic and real-world datasets.

\end{abstract}

\section{Introduction} \label{sec:intro}
In this study, we consider the problem of identifying unnecessary samples and features in a class of supervised learning problems within dynamically changing environments.
Identifying unnecessary samples/features offers several benefits.
It helps in decreasing the storage space required for keeping the training data for updating the machine learning (ML) models in the future.
Moreover, in situations demanding real-time adaptation of ML models to quick environmental changes, the use of fewer samples/features enables more efficient learning.

Our basic idea to tackle this problem is to effectively combine \emph{distributionally robust (DR)} learning and \emph{safe screening (SS)}. 
DR learning is a ML paradigm that focuses on developing models robust to variations in the data distribution, providing performance guarantees across different distributions (see, e.g., \cite{chen2021distributionally}).
%
%
On the other hand, SS refers to sparse optimization techniques that can identify irrelevant samples/features before model training, ensuring computational efficiency by avoiding unnecessary computations on certain samples/features which do not contribute to the final solution~\cite{ghaoui2012safe,ogawa2013safe}.
The key technical idea of SS is to identify a bound of the optimal solution before solving the optimization problem.
This allows for the identification of unnecessary samples/features, even without knowing the optimal solution.

As a specific scenario of dynamically changing environment, we consider \emph{covariate shift} setting \cite{shimodaira2000improving,sugiyama2007covariate} with unknown test distribution.
In this setting, the distribution of input features in the training data may undergo changes in the test phase, yet the actual nature of these changes remains unknown. 
A ML problem (e.g., regression/classification problem) in covariate shift setting can be formulated as a \emph{weighted empirical risk minimization (weighted ERM)} problem, where weights are assigned based on the density ratio of each sample in the training and test distributions.
Namely, by assigning higher weights to training samples that are important in the test distribution, the model can focus on learning from relevant samples and mitigate the impact of distribution differences between the training and the test phases.
If the distribution during the test phase is known, the weights can be uniquely fixed.
However, if the test distribution is unknown, it is necessary to solve a weighted ERM problem with unknown weights. 

Our main contribution is to propose a DRSS method for covariate shift setting with unknown test distribution.
The proposed method can identify unnecessary samples/features regardless of how the distribution changes within a certain range in the test phase.
To address this problem, we extend the existing SS methods in two stages.
The first is to extend the SS for ERM so that it can be applied to weighted ERM.
The second is to further extend the SS so that it can be applied to weighted ERM when the weights are unknown.
While the first extension is relatively straightforward, the second extension presents a non-trivial technical challenge (Figure \ref{fig:concept}).
To overcome this challenge, we derive a novel bound of the optimal solutions of the weighted ERM problem, which properly accounts for the uncertainty in weights stemming from the uncertainty of the test distribution.

In this study, we consider DRSS for samples in sample-sparse models such as SVM~\cite{cortes1995support}, and that for features for feature-sparse models such as Lasso~\cite{tibshirani1996regression}. 
We denote the DRSS for samples as distributionally robust safe \emph{sample} screening (DRSsS) and that for features as distributionally robust safe \emph{feature} screening (DRSfS), respectively. 

Our contributions in this study are summarized as follows.
First, by effectively combining DR and SS, we introduce a framework for identifying unnecessary samples/features under dynamically changing uncertain environment.
Second, We consider a DR covariate-shift setting where the input distribution of an ERM problem changes within a certain range.
In this setting, we propose a novel method called DRSS method that can identify samples/features that are guaranteed not to affect the optimal solution, regardless of how the distribution changes within the specified range.
Finally, through numerical experiments, we verify the effectiveness of the proposed DRSS method. 
Although the DRSS method is developed for convex ERM problems, in order to demonstrate the applicability to deep learning models, we also present results where the DRSS method is applied in a problem setting where the final layer of the model is fine-tuned according to changes in the test distribution.

\begin{figure}[t]
\includegraphics[width=0.9\hsize]{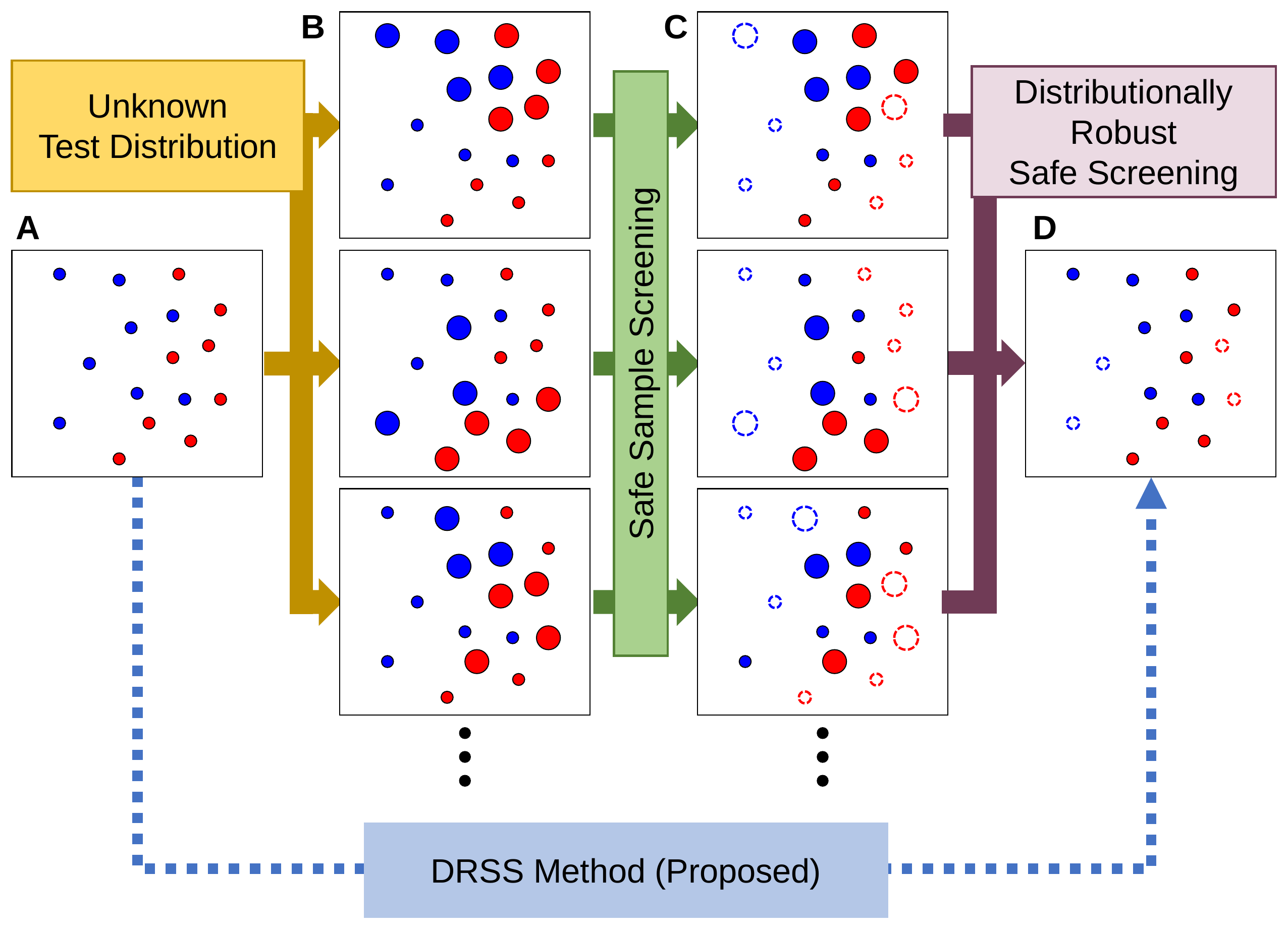}
\vspace{-1em}
\caption{
 Schematic illustration of the proposed Distributionally Robust Safe Screening (DRSS) method.
 Panel {\bf A} displays the training samples, each assigned equal weight, as indicated by the uniform size of the points.
 Panel {\bf B} depicts various unknown test distributions, highlighting how the significance of training samples varies with different realizations of the test distribution.
 Panel {\bf C} shows the outcomes of safe sample screening (SsS) across multiple realizations of test distributions.
 Finally, Panel {\bf D} presents the results of the proposed DRSS method, demonstrating its capability to identify redundant samples regardless of the observed test distribution.
 }
\label{fig:concept}
\end{figure}

\subsection{Related Works} \label{sec:related}
%
The DR setting has been explored in various ML problems, aiming to enhance model robustness against data distribution variations.
A DR learning problem is typically formulated as a worst-case optimization problem since the goal of DR learning is to ensure model performance under the worst-case data distribution within a specified range.
Hence, a variety of optimization techniques tailored to DR learning have been investigated within both the ML and optimization communities~\cite{goh2010distributionally,delage2010distributionally,chen2021distributionally}.
The proposed DRSS method is one of such DR learning methods, focusing specifically on the problem of sample/feature deletion.
The ability to identify irrelevant samples/features is of practical significance.
For example, in the context of continual learning (see, e.g., \cite{wang2022memory}), it is crucial to effectively manage data by selectively retaining and discarding samples/features, especially in anticipation of changes in future data distributions.
Incorrect deletion of essential data can lead to \emph{catastrophic forgetting}~\cite{kirkpatrick2017overcoming}, a phenomenon where a ML model, after being trained on new data, quickly loses information previously learned from older datasets.
The proposed DRSS method tackles this challenge by identifying samples/features that, regardless of future data distribution shifts, will not have any influence on all possible newly trained model in the future. 

SS refers to optimization techniques in sparse learning that identify and exclude irrelevant samples or features from the learning process.
SS can reduce computational cost without changing the final trained model.
Initially, SfS was introduced by \cite{ghaoui2012safe} for the Lasso.
Subsequently, SsS was proposed by \cite{ogawa2013safe} for the SVM.
Among various SS methods developed so far, the most commonly used is based on the duality gap~\cite{fercoq2015mind,ndiaye2015gap}.
Our proposed DRSS method also adopts this approach.
Over the past decade, SS has seen diverse developments, including methodological improvements and expanded application scopes~\cite{okumura2015quick,shibagaki2016simultaneous,nakagawa2016safe,ren2018safe,zhao2019improved,zhai2020safe,wang2022safe,yoshida2023efficient}.
Unlike other SS studies that primarily focused on reducing computational costs, this study adopts SS for a different purpose.
We employ SS across scenarios where data distribution varies within a defined range, aiming to discard unnecessary samples/features.
To our knowledge, no existing studies have utilized SS within the DR learning framework.

\section{Preliminaries} \label{sec:preliminaries}

Notations used in this paper are described in Table \ref{tab:definitions}.

\begin{table}[t]
\caption{Notations used in the paper. $\mathbb{R}$: all real numbers, $\mathbb{N}$: all positive integers, $n, m, p\in\mathbb{N}$: integers, $f: \mathbb{R}^n\to\mathbb{R}\cup\{+\infty\}$: convex function, $M\in\mathbb{R}^{n\times m}$: matrix, $\bm v\in\mathbb{R}^n$: vector.}
\label{tab:definitions}
\begin{tabular}{ll}
\hline
\multicolumn{2}{l}{$m_{ij} \in \mathbb{R}$ (small case of matrix variable)} \\
	& the element at the $i^\mathrm{th}$ row and \\
	& the $j^\mathrm{th}$ column of $M$ \\
\multicolumn{2}{l}{$v_i \in \mathbb{R}$ (nonbold font of vector variable)} \\
	& the $i^\mathrm{th}$ element of $\bm v$\\
$M_{i:}\in\mathbb{R}^{1\times n}$ & the $i^\mathrm{th}$ row of $M$ \\
$M_{:j}\in\mathbb{R}^{m\times 1}$ & the $j^\mathrm{th}$ column of $M$ \\
$[n]$ & $\{1, 2, \dots, n\}$ \\
$\mathbb{R}_{\geq 0}$ & all nonnegative real numbers \\
$\otimes$ & elementwise product \\
$\mathrm{diag}(\bm v) \in\mathbb{R}^{n\times n}$
	& diagonal matrix; $(\mathrm{diag}(\bm v))_{ii} = v_i$ \\
	& and $(\mathrm{diag}(\bm v))_{ij} = 0$ ($i\neq j$) \\
$\bm v\sqtimes M \in\mathbb{R}^{n\times m}$ & $\mathrm{diag}(\bm v) M$ \\
$\bm 0_n \in \mathbb{R}^n$ & $[0, 0, \dots, 0]^\top$ (vector of size $n$) \\
$\bm 1_n \in \mathbb{R}^n$ & $[1, 1, \dots, 1]^\top$ (vector of size $n$) \\
$\|\bm v\|_p \in \mathbb{R}_{\geq 0}$ & $(\sum_{i=1}^n v_i^p)^{1/p}$ (\emph{$p$-norm}) \\
$\partial f(\bm v) \subseteq \mathbb{R}^n$
	& all $\bm g\in\mathbb{R}^n$ s.t. ``for any $\bm v^\prime\in\mathbb{R}^n$, \\
	& $f(\bm v^\prime) - f(\bm v)\geq \bm g^\top(\bm v^\prime - \bm v)$'' \\
	&  (\emph{subgradient}) \\
${\cal Z}[f] \subseteq\mathbb{R}^n$ & $\{ \bm v^\prime\in\mathbb{R}^n \mid \partial f(\bm v^\prime) = \{\bm 0_n\} \}$ \\
$f^*(\bm v) \in \mathbb{R}\cup\{+\infty\}$
	& $\sup_{\bm v^\prime\in\mathbb{R}^n} (\bm v^\top \bm v^\prime - f(\bm v^\prime))$ \\
	& (\emph{convex conjugate}) \\
``$f$ is $\kappa$-strongly
	& $f(\bm v) - \kappa\|\bm v\|_2^2$ is convex with \\
	\quad convex'' ($\kappa>0$)
		&\quad respect to $\bm v$ \\
``$f$ is $\mu$-smooth''
	& $\|f(\bm v) - f(\bm v^\prime)\|_2\leq \mu\|\bm v - \bm v^\prime\|_2$ \\
	\quad ($\mu > 0$)
		&\quad for any $\bm v, \bm v^\prime\in\mathbb{R}^n$ \\
\hline
\end{tabular}
\end{table}

\subsection{Weighted Regularized Empirical Risk Minimization (Weighted RERM) for Linear Prediction} \label{sec:weighted-RERM}

We mainly assume the weighted regularized empirical risk minimization (weighted RERM) for linear prediction.
This may include kernelized versions, which are discussed in Appendix \ref{app:kernelized}.
Suppose that we learn the model parameters as linear prediction coefficients, that is,
learn $\bm\beta^{*(\bm w)}\in\mathbb{R}^d$ such that the outcome for a sample $\bm x\in\mathbb{R}^d$ is predicted as $\bm x^\top\bm\beta^{*(\bm w)}$.

\begin{definition}\label{def:WRERM}
Given $n$ training samples of $d$-dimensional input variables, scalar output variables
and scalar sample weights, denoted by $X\in\mathbb{R}^{n\times d}$, $\bm y\in\mathbb{R}^n$ and
$\bm w\in\mathbb{R}_{\geq 0}^n$, respectively,
the training computation of weighted RERM for linear prediction is formulated as follows:
\begin{align}
& \bm\beta^{*(\bm w)} := \argmin_{\bm\beta\in\mathbb{R}^d} P_{\bm w}(\bm\beta), \nonumber
	\quad\text{where} \\
& P_{\bm w}(\bm\beta) := \sum_{i=1}^n w_i \ell_{y_i}(\check{X}_{i:}\bm\beta) + \rho(\bm\beta). \label{eq:primal}
\end{align}
Here, $\ell_y: \mathbb{R}\to\mathbb{R}$ is a convex {\em loss function}\footnote{For $\ell_y(t)$, we assume that only $t$ is a variable of the function ($y$ is assumed to be a constant) when we take its subgradient or convex conjugate.}, $\rho: \mathbb{R}^d\to\mathbb{R}$ is a convex {\em regularization function}, and $\check{X}\in\mathbb{R}^{n\times d}$ is a matrix calculated from $X$ and $\bm y$ and determined depending on $\ell$.
In this paper, unless otherwise noted, we consider binary classifications ($\bm y\in\{-1, +1\}^n$)
with $\check{X} := \bm y \sqtimes X$.
For regressions ($\bm y\in\mathbb{R}^n$) we usually set $\check{X} := X$.
\end{definition}
\begin{remark} \label{rem:intercept}
We add that, we adopt the formulation $X_{:d} = \bm 1_n$ so that $\beta^{*(\bm w)}_d$ (the last element) represents the common coefficient for any sample (called the \emph{intercept}).
\end{remark}
Since $\ell$ and $\rho$ are convex, we can easily confirm that $P_{\bm w}(\bm\beta)$ is convex with respect to $\bm\beta$.

Applying {\em Fenchel's duality theorem} (Appendix \ref{app:fenchel}), we have the following {\em dual problem} of \eqref{eq:primal}:
\begin{align}
& \bm\alpha^{*(\bm w)} := \argmax_{\bm\alpha\in\mathbb{R}^n} D_{\bm w}(\bm\alpha),
	\quad\text{where} \nonumber \\
& D_{\bm w}(\bm\alpha) := \label{eq:dual} \\
& -\sum_{i=1}^n w_i \ell^*_{y_i}(-\gamma_i \alpha_i) - \rho^*(((\bm\gamma \otimes \bm w)\sqtimes\check{X})^\top \bm\alpha), \nonumber
\end{align}
where $\bm\gamma$ is a positive-valued vector.
The relationship between the original problem \eqref{eq:primal} (called the \emph{primal} problem) and the dual problem \eqref{eq:dual} are described as follows:
\begin{align}
& P_{\bm w}(\bm\beta^{*(\bm w)}) = D_{\bm w}(\bm\alpha^{*(\bm w)}), \label{eq:strong-duality}\\
& \bm\beta^{*(\bm w)} \in \partial\rho^*(((\bm\gamma \otimes \bm w)\sqtimes\check{X})^\top \bm\alpha^{*(\bm w)}), \label{eq:KKT-dual2primal}\\
& \forall i\in[n]:\quad -\gamma_i\alpha^{*(\bm w)}_i \in \partial\ell_{y_i}(\check{X}_{i:}\bm\beta^{*(\bm w)}). \label{eq:KKT-primal2dual}
\end{align}

\subsection{Sparsity-inducing Loss Functions and Regularization Functions}

In weighted RERM, we call that a loss function $\ell$ induces \emph{sample-sparsity} if elements in $\bm\alpha^{*(\bm w)}$ are easy to become zero.
Due to \eqref{eq:KKT-primal2dual}, this can be achieved by $\ell$ such that $\{ t \in\mathbb{R} \mid 0 \in \partial\ell_{y}(t) \}$ is not a point but an interval.

Similarly, we call that a regularization function $\rho$ induces \emph{feature-sparsity} if elements in $\bm\beta^{*(\bm w)}$ are easy to become zero.
Due to \eqref{eq:KKT-dual2primal}, this can be achieved by $\rho$ such that $\{ \bm v \in\mathbb{R}^d \mid \exists j\in[d-1]:~0 \in [\partial\rho^*(\bm v)]_j \}$ is not a point but a region.

For example, the \emph{hinge loss} $\ell_{y}(t) = \max\{0, 1-t\}$ ($y\in\{-1, +1\}$) is a sample-sparse loss function since $\{ t \in\mathbb{R} \mid 0 \in \partial\ell_{y}(t) \} = [1, +\infty)$.
Similarly, the \emph{L1-regularization} $\rho(\bm v) = \lambda \sum_{j=1}^{d-1} |v_j|$ ($\lambda > 0$: hyperparameter) is a feature-sparse regularization function since
$\{ \bm v \in\mathbb{R}^d \mid \exists j\in[d-1]:~0 \in [\partial\rho^*(\bm v)]_j \} = \{ \bm v \in\mathbb{R}^d \mid \exists j\in[d-1]:~|v_j|\leq\lambda,~v_d = 0 \}$.
See Section \ref{sec:DRSS-examples} for examples of using them.


\section{Distributionally Robust Safe Screening}

In this section we show DRSS rules for weighted RERM with two steps.
First, in Sections \ref{sec:safe-sample-screening} and \ref{sec:safe-feature-screening},
we show SS rules for weighted RERM but not DR setup.
To do this, we extended existing SS rules in \cite{ndiaye2015gap,shibagaki2016simultaneous}.
Then we derive DRSS rules in Section \ref{sec:DRSS}.

\subsection{(Non-DR) Safe Sample Screening} \label{sec:safe-sample-screening}

We consider identifying training samples that do not affect the training result $\bm\beta^{*(\bm w)}$.
Due to the relationship \eqref{eq:KKT-dual2primal}, if there exists $i\in[n]$ such that $\alpha^{*(\bm w)}_i = 0$,
then the $i^\mathrm{th}$ row (sample) in $\check{X}$ does not affect $\bm\beta^{*(\bm w)}$.
However, since computing $\bm\alpha^{*(\bm w)}$ is as costly as $\bm\beta^{*(\bm w)}$, it is difficult to use the relationship as it is.
To solve the problem, the SsS first considers identifying the possible region ${\cal B}^{*(\bm w)}\subset\mathbb{R}^d$ such that $\bm\beta^{*(\bm w)}\in{\cal B}^{*(\bm w)}$ is assured.
Then, with ${\cal B}^{*(\bm w)}$ and \eqref{eq:KKT-primal2dual}, we can conclude that the $i^\mathrm{th}$ training sample do not affect the training result $\bm\beta^{*(\bm w)}$ if
$\bigcup_{\bm\beta\in{\cal B}^{*(\bm w)}} \partial\ell_{y_i}(\check{X}_{i:}\bm\beta) = \{0\}$.

First we show how to compute ${\cal B}^{*(\bm w)}$.
In this paper we adopt the computation methods that is available when the regularization function $\rho$ in $P_{\bm w}$ (and also $P_{\bm w}$ itself) of \eqref{eq:primal} are strongly convex.

\begin{lemma}\label{lem:gap-sphere-primal}
Suppose that $\rho$ in $P_{\bm w}$ (and also $P_{\bm w}$ itself) of \eqref{eq:primal} are $\kappa$-strongly convex.
Then, for any $\hat{\bm\beta}\in\mathbb{R}^d$ and $\hat{\bm\alpha}\in\mathbb{R}^n$, we can assure $\bm\beta^{*(\bm w)}\in{\cal B}^{*(\bm w)}$ by taking
\begin{align*}
& {\cal B}^{*(\bm w)} := \left\{ \bm\beta \;\middle|\; \| \bm\beta - \hat{\bm\beta} \|_2 \leq r(\bm w, \bm\gamma, \kappa, \hat{\bm\beta}, \hat{\bm\alpha}) \right\}, \\
& \text{where}\quad
	r(\bm w, \bm\gamma, \kappa, \hat{\bm\beta}, \hat{\bm\alpha}) := \sqrt{\frac{2}{\kappa}[P_{\bm w}(\hat{\bm\beta}) - D_{\bm w}(\hat{\bm\alpha})]}.
\end{align*}
\end{lemma}
The proof is presented in Appendix \ref{app:gap-sphere-primal}.
The amount $P_{\bm w}(\hat{\bm\beta}) - D_{\bm w}(\hat{\bm\alpha})$ is known as the \emph{duality gap}, which must be nonnegative due to \eqref{eq:strong-duality}.
So we obtain the following \emph{gap safe sample screening rule} from Lemma \ref{lem:gap-sphere-primal}:
\begin{lemma}\label{lem:gap-sample-screening}
Under the same assumptions as Lemma \ref{lem:gap-sphere-primal}, $\alpha_i^{*(\bm w)} = 0$ is assured
(i.e., the $i^\mathrm{th}$ training sample does not affect the training result $\bm\beta^{*(\bm w)}$)
if there exists $\hat{\bm\beta}\in\mathbb{R}^d$ and $\hat{\bm\alpha}\in\mathbb{R}^n$ such that
\begin{align*}
& [\check{X}_{i:}\hat{\bm\beta} - \|\check{X}_{i:}\|_2 r(\bm w, \bm\gamma, \kappa, \hat{\bm\beta}, \hat{\bm\alpha}), \\
& \quad \check{X}_{i:}\hat{\bm\beta} + \|\check{X}_{i:}\|_2 r(\bm w, \bm\gamma, \kappa, \hat{\bm\beta}, \hat{\bm\alpha})] \subseteq {\cal Z}[\ell_{y_i}].
\end{align*}
\end{lemma}
The proof is presented in Appendix \ref{app:gap-sample-screening}.

\subsection{(Non-DR) Safe Feature Screening} \label{sec:safe-feature-screening}

We consider identifying $j\in[d]$ such that $\beta_j^{*(\bm w)} = 0$, that is, 
identifying that the $j^\mathrm{th}$ feature is not used in the prediction,
even when the sample weights $\bm w$ are changed.

For simplicity, suppose that the regularization function $\rho$ is decomposable, that is, $\rho$ is represented as $\rho(\bm\beta) := \sum_{j=1}^d \sigma_j(\beta_j)$ ($\sigma_1, \sigma_2, \dots, \sigma_d$: $\mathbb{R}\to\mathbb{R}$).
Then, since $\rho^*(\bm v) = \sum_{j=1}^d \sigma^*_j(v_j)$ and
therefore $[\partial\rho^*(\bm v)]_j = \partial\sigma^*_j(v_j)$,
from \eqref{eq:KKT-dual2primal} we have
\begin{align*}
\beta^{*(\bm w)}_j &\in \partial\sigma^*_j((\bm\gamma \otimes \bm w\otimes\check{X}_{:j})^\top \bm\alpha^{*(\bm w)}) \\
& = \partial\sigma^*_j(\check{\check{X}}_{:j}^{(\bm\gamma,\bm w)\top} \bm\alpha^{*(\bm w)}), \\
\text{where} & \quad \check{\check{X}}_{:j}^{(\bm\gamma,\bm w)} := \bm\gamma \otimes \bm w \otimes \check{X}_{:j}.
\end{align*}
If we know $\bm\alpha^{*(\bm w)}$, we can identify whether $\beta_j^{*(\bm w)} = 0$ holds.
However, like SsS (Section \ref{sec:safe-sample-screening}), we would like to check the condition without computing $\bm\alpha^{*(\bm w)}$ or $\bm\beta^{*(\bm w)}$.

So, like SsS,
SfS first considers identifying the possible region ${\cal A}^{*(\bm w)}\subset\mathbb{R}^n$
such that $\bm\alpha^{*(\bm w)}\in{\cal A}^{*(\bm w)}$ is assured.
Then we can conclude that $\beta^{*(\bm w)}_j = 0$ is assured if
$\bigcup_{\bm\alpha\in{\cal A}^{*(\bm w)}} \partial\sigma^*_j(\check{\check{X}}_{:j}^{(\bm\gamma,\bm w)\top} \bm\alpha) = \{0\}$.

Then we show how to compute ${\cal A}^{*(\bm w)}$.
With Lemma \ref{lem:strong-convexity-sphere}, we can calculate ${\cal A}^{*(\bm w)}$
as follows, if the loss function $\ell_y$ in $P_{\bm w}$ of \eqref{eq:primal} is smooth:
\begin{lemma}\label{lem:gap-sphere-dual}
Suppose that $\ell_y$ in $P_{\bm w}$ of \eqref{eq:primal} is $\mu$-smooth.
Then, for any $\hat{\bm\beta}\in\mathbb{R}^d$ and $\hat{\bm\alpha}\in\mathbb{R}^n$, we can assure $\bm\alpha^{*(\bm w)}\in{\cal A}^{*(\bm w)}$ by taking
\begin{align*}
& {\cal A}^{*(\bm w)} := \left\{ \bm\alpha \;\middle|\; \| \bm\alpha - \hat{\bm\alpha} \|_2 \leq \bar{r}(\bm w, \bm\gamma, \mu, \hat{\bm\beta}, \hat{\bm\alpha}) \right\}, \\
& \text{where}\quad
	\bar{r}(\bm w, \bm\gamma, \mu, \hat{\bm\beta}, \hat{\bm\alpha}) := \\
	& \phantom{\text{where}}\quad
	\sqrt{\frac{2 \mu}{\min_{i\in[n]} w_i \gamma_i^2}[P_{\bm w}(\hat{\bm\beta}) - D_{\bm w}(\hat{\bm\alpha})]}.
\end{align*}
\end{lemma}
The proof is presented in Appendix \ref{app:gap-sphere-dual}.
Similar to Lemma \ref{lem:gap-sample-screening}, we obtain the \emph{gap safe feature screening rule}
from Lemma \ref{lem:gap-sphere-dual}:
\begin{lemma}\label{lem:gap-feature-screening}
Under the same assumptions as Lemma \ref{lem:gap-sphere-dual}, $\beta_j^{*(\bm w)} = 0$ is assured
(i.e., the $j^\mathrm{th}$ feature does not affect prediction results)
if there exists $\hat{\bm\beta}\in\mathbb{R}^d$ and $\hat{\bm\alpha}\in\mathbb{R}^n$ such that
\begin{align*}
& [\check{\check{X}}_{:j}^{(\bm\gamma,\bm w)\top} \hat{\bm\alpha} - \|\check{\check{X}}_{:j}^{(\bm\gamma,\bm w)}\|_2 \bar{r}(\bm w, \bm\gamma, \mu, \hat{\bm\beta}, \hat{\bm\alpha}), \\
& \quad\check{\check{X}}_{:j}^{(\bm\gamma,\bm w)\top} \hat{\bm\alpha} + \|\check{\check{X}}_{:j}^{(\bm\gamma,\bm w)}\|_2 \bar{r}(\bm w, \bm\gamma, \mu, \hat{\bm\beta}, \hat{\bm\alpha})] \subseteq {\cal Z}[\sigma^*_j].
\end{align*}
\end{lemma}
The proof is almost same as Lemma \ref{lem:gap-sample-screening}.

\subsection{Application to Distributionally Robust Setup} \label{sec:DRSS}

In Sections \ref{sec:safe-sample-screening} and \ref{sec:safe-feature-screening} we showed the conditions when samples or features are screened out.
In this section we show how to use the conditions for the change of sample weights $\bm w$.

\begin{definition}[weight-changing safe screening (WCSS)] \label{def:safe-screening-specific-change}
Given $X\in\mathbb{R}^{n\times d}$, $\bm y\in\mathbb{R}^n$,
$\tilde{\bm w}\in\mathbb{R}_{\geq 0}^n$ and $\bm w\in\mathbb{R}_{\geq 0}^n$,
suppose that $\bm\beta^{*(\tilde{\bm w})}$ in Definition \ref{def:WRERM} (and also $\bm\alpha^{*(\tilde{\bm w})}$) are already computed, but $\bm\beta^{*(\bm w)}$ not.
Then \emph{WCSsS (resp. WCSfS) from $\tilde{\bm w}$ to $\bm w$} is defined as finding $i\in[n]$ satisfying Lemma \ref{lem:gap-sample-screening} (resp. $j\in[d-1]$ satisfying Lemma \ref{lem:gap-feature-screening}).
\end{definition}

\begin{definition}[Distributionally robust safe screening (DRSS)] \label{def:safe-screening-robust-change}
Given $X\in\mathbb{R}^{n\times d}$, $\bm y\in\mathbb{R}^n$,
$\tilde{\bm w}\in\mathbb{R}_{\geq 0}^n$ and ${\cal W}\subset\mathbb{R}_{\geq 0}^n$,
suppose that $\bm\beta^{*(\tilde{\bm w})}$ in Definition \ref{def:WRERM} (and also $\bm\alpha^{*(\tilde{\bm w})}$) are already computed.
Then the \emph{DRSsS (resp. DRSfS) for ${\cal W}$} is defined as finding $i\in[n]$ satisfying Lemma \ref{lem:gap-sample-screening} (resp. $j\in[d-1]$ satisfying Lemma \ref{lem:gap-feature-screening}) for any $\bm w\in{\cal W}$.
\end{definition}

For Definition \ref{def:safe-screening-specific-change}, we have only to apply SS rules in
Lemma \ref{lem:gap-sample-screening} or \ref{lem:gap-feature-screening} by setting
$\hat{\bm\beta}\gets\bm\beta^{*(\tilde{\bm w})}$ and $\hat{\bm\alpha}\gets\bm\alpha^{*(\tilde{\bm w})}$.
On the other hand, for Definition \ref{def:safe-screening-robust-change}, we need to maximize or minimize
the interval in Lemma \ref{lem:gap-sample-screening} or \ref{lem:gap-feature-screening} in $\bm w\in{\cal W}$.
\begin{theorem} \label{thm:safe-screening-robust}
The DRSsS rule for ${\cal W}$ is calculated as:
\begin{align*}
& [\check{X}_{i:}\bm\beta^{*(\tilde{\bm w})} - \|\check{X}_{i:}\|_2 R,
	\check{X}_{i:}\bm\beta^{*(\tilde{\bm w})} + \|\check{X}_{i:}\|_2 R] \subseteq {\cal Z}[\ell_{y_i}],
\end{align*}
where
$R := \max_{\bm w\in{\cal W}} r(\bm w, \bm\gamma, \kappa, \bm\beta^{*(\tilde{\bm w})}, \bm\alpha^{*(\tilde{\bm w})})$.

Similarly, the DRSfS rule for ${\cal W}$ is calculated as:
\begin{align*}
& [ \underline{L} - N \overline{R},
	\overline{L} + N \overline{R}] \subseteq {\cal Z}[\sigma^*_j],
	\quad\text{where} \\
& \underline{L} := \min_{\bm w\in{\cal W}}\check{\check{X}}_{:j}^{(\bm\gamma,\bm w)\top} \bm\alpha^{*(\tilde{\bm w})}
	= \min_{\bm w\in{\cal W}} (\bm\gamma \otimes \check{X}_{:j} \otimes \bm\alpha^{*(\tilde{\bm w})})^\top \bm w, \\
& \overline{L} := \max_{\bm w\in{\cal W}}\check{\check{X}}_{:j}^{(\bm\gamma,\bm w)\top} \bm\alpha^{*(\tilde{\bm w})}
	= \max_{\bm w\in{\cal W}} (\bm\gamma \otimes \check{X}_{:j} \otimes \bm\alpha^{*(\tilde{\bm w})})^\top \bm w, \\
& N := \max_{\bm w\in{\cal W}}\|\check{\check{X}}_{:j}^{(\bm\gamma,\bm w)}\|_2
	= \sqrt{\max_{\bm w\in{\cal W}} \| \bm w\otimes\bm\gamma\otimes\check{X}_{:j} \|_2^2 }, \\
& \overline{R} := \max_{\bm w\in{\cal W}}\bar{r}(\bm w, \bm\gamma, \mu, \bm\beta^{*(\tilde{\bm w})}, \bm\alpha^{*(\tilde{\bm w})}).
\end{align*}
\end{theorem}

Thus, solving the maximizations and/or minimizations in Theorem \ref{thm:safe-screening-robust} provides DRSsS and DRSfS rules.
However, how to solve it largely depends on the choice of $\ell$, $\rho$ and ${\cal W}$.
In Section \ref{sec:DRSS-examples} we show specific calculations of Theorem \ref{thm:safe-screening-robust}
for some typical setups.

\section{DRSS for Typical ML Setups} \label{sec:DRSS-examples}

In this section we show DRSS rules derived in Section \ref{sec:DRSS}
for two typical ML setups:
DRSsS for L1-loss L2-regularized SVM (Section \ref{sec:l1loss-l2reg-svm}) and
DRSfS for L2-loss L1-regularized SVM (Section \ref{sec:l2loss-l1reg-svm})
under ${\cal W} := \{ \bm w \mid \|\bm w - \tilde{\bm w}\|_2\leq S \}$.

In the processes, we need to solve constrained maximizations of convex functions.
Although maximizations of convex functions are not easy in general (minimizations are easy),
we show that the maximizations need in the processes can be algorithmically solved
in Section \ref{sec:maximize-convex-quadratic}.

\subsection{DRSsS for L1-loss L2-regularized SVM} \label{sec:l1loss-l2reg-svm}

L1-loss L2-regularized SVM is a sample-sparse model for binary classification ($\bm y\in\{-1,+1\}^n$)
that satisfies the preconditions to apply SsS (Lemma \ref{lem:gap-sphere-primal}).
Detailed calculations are presented in Appendix \ref{app:details-l1loss-l2reg-svm}.

For L1-loss L2-regularized SVM, we set $\rho$ and $\ell$ as:
\begin{align*}
& \rho(\bm\beta) := \frac{\lambda}{2}\|\bm\beta\|_2^2 \quad (\lambda > 0:~\text{hyperparameter}), \\
& \ell_y(t) := \max\{0, 1 - t\} \quad(\text{where}~y\in\{-1, +1\}).
\end{align*}
Then $\rho$ is $\lambda$-strongly convex.
Setting $\bm\gamma = \bm 1_n$, the dual objective function is described as
\begin{align}
& D_{\bm w}(\bm\alpha) = \nonumber\\
& \left\{ \begin{array}{r}
	\sum_{i=1}^n w_i \alpha_i - \frac{1}{2\lambda} \bm\alpha^\top (\bm w\sqtimes\check{X})(\bm w\sqtimes\check{X})^\top \bm\alpha,\\
		(\forall i\in[n]: 0 \leq \alpha_i \leq 1) \\
	-\infty. \hfill (\text{otherwise})
	\end{array} \right.
	\label{eq:dual-l1loss-l2reg-svm}
\end{align}
Here, in the viewpoint of minimization, we may consider this problem as a maximization with the constraint ``$\forall i\in[n]: 0 \leq \alpha_i \leq 1$''.

Optimality conditions \eqref{eq:KKT-dual2primal} and \eqref{eq:KKT-primal2dual} are described as:
\begin{align}
& \bm\beta^{*(\bm w)} = \frac{1}{\lambda}(\bm w\sqtimes\check{X})^\top \bm\alpha^{*(\bm w)},
	\label{eq:KKT-dual2primal-l1loss-l2reg-svm}\\
& \forall i\in[n]:\quad \alpha^{*(\bm w)}_i \in \begin{cases}
	\{ 1 \}, & (\check{X}_{i:}\bm\beta^{*(\bm w)} \leq 1) \\
	[0, 1], & (\check{X}_{i:}\bm\beta^{*(\bm w)} = 1) \\
	\{ 0 \}. & (\check{X}_{i:}\bm\beta^{*(\bm w)} \geq 1)
	\end{cases}
	\label{eq:KKT-primal2dual-l1loss-l2reg-svm}
\end{align}

Noticing that ${\cal Z}[\ell_{y_i}] = (1, +\infty)$,
by Theorem \ref{thm:safe-screening-robust},
the DRSsS rule for ${\cal W}$
is calculated as:
\begin{align}
& \check{X}_{i:}\bm\beta^{*(\tilde{\bm w})} - \|\check{X}_{i:}\|_2 \max_{\bm w\in{\cal W}} r(\bm w, \bm\gamma, \kappa, \bm\beta^{*(\tilde{\bm w})}, \bm\alpha^{*(\tilde{\bm w})}) > 1, \nonumber\\
& \quad\text{where} \label{eq:safe-sample-screening-l1loss-l2reg-svm} \\
& r(\bm w, \bm\gamma, \kappa, \bm\beta^{*(\tilde{\bm w})}, \bm\alpha^{*(\tilde{\bm w})}) \nonumber\\
& := \sqrt{\frac{2}{\kappa}[P_{\bm w}(\bm\beta^{*(\tilde{\bm w})}) - D_{\bm w}(\bm\alpha^{*(\tilde{\bm w})})]}, \nonumber \\
& P_{\bm w}(\bm\beta^{*(\tilde{\bm w})}) - D_{\bm w}(\bm\alpha^{*(\tilde{\bm w})}) \nonumber\\
& := \sum_{i=1}^n w_i [ \ell_{y_i}(\check{X}_{i:}\bm\beta^{*(\tilde{\bm w})}) - \alpha^{*(\tilde{\bm w})}_i] + \lambda\|\bm\beta^{*(\tilde{\bm w})}\|_2^2 \nonumber\\
& \phantom{:=} +\frac{1}{2\lambda} \bm w^\top (\bm\alpha^{*(\tilde{\bm w})}\sqtimes\check{X}) (\bm\alpha^{*(\tilde{\bm w})}\sqtimes\check{X})^\top \bm w . \nonumber
\end{align}
Here, we can find that $P_{\bm w}(\bm\beta^{*(\tilde{\bm w})}) - D_{\bm w}(\bm\alpha^{*(\tilde{\bm w})})$, which we need to maximize in reality, is the sum of linear function and convex quadratic function with respect to $\bm w\in{\cal W}$. (Since $(\bm\alpha^{*(\tilde{\bm w})}\sqtimes\check{X}) (\bm\alpha^{*(\tilde{\bm w})}\sqtimes\check{X})^\top$ is positive semidefinite, we know that it is convex quadratic).
Although constrained maximization of a convex function is difficult in general,
for this case we can algorithmically maximize it (Section \ref{sec:maximize-convex-quadratic}).

\subsection{DRSfS for L2-loss L1-regularized SVM} \label{sec:l2loss-l1reg-svm}

L2-loss L1-regularized SVM is a feature-sparse model for binary classification ($\bm y\in\{-1,+1\}^n$)
that satisfies the preconditions to apply SfS (Lemma \ref{lem:gap-sphere-dual}).
Detailed calculations are presented in Appendix \ref{app:details-l2loss-l1reg-svm}.

For L2-loss L1-regularized SVM, we set $\sigma_j$ (and consequently $\rho$) and $\ell$ as:
\begin{align*}
& \forall j\in[d-1]:~\sigma_j(\beta_j) := \lambda |\beta_j| \quad (\lambda > 0:~\text{hyperparameter}), \\
& \sigma_d(\beta_d) := 0, \\
& \ell_y(t) := (\max\{0, 1 - t\})^2 \quad(\text{where}~y\in\{-1, +1\}).
\end{align*}
Notice that $\sigma_d(\beta_d)$ is not defined as $\lambda|\beta_d|$ but $0$: we rarely regularize the intercept with L1-regularization.

Setting $\bm\gamma = \lambda \bm 1_n$,
the dual objective function is described as
\begin{align}
& D_{\bm w}(\bm\alpha) = \begin{cases}
	-\lambda\sum_{i=1}^n w_i \frac{\lambda\alpha^2_i - 4\alpha_i}{4}, & (\text{\eqref{eq:l2loss-l1reg-constraint-1}--\eqref{eq:l2loss-l1reg-constraint-3}~are~met}) \\
	-\infty, & (\text{otherwise})
	\end{cases}
	\label{eq:dual-l2loss-l1reg-svm} \\
& \text{where}\quad
	\alpha_i \geq 0, \label{eq:l2loss-l1reg-constraint-1}\\
& \phantom{\text{where}}\quad
	\forall j\in[d-1]:~| (\bm w\otimes\check{X}_{:j})^\top \bm\alpha | \leq 1, \label{eq:l2loss-l1reg-constraint-2}\\
& \phantom{\text{where}}\quad
	(\bm w\otimes\check{X}_{:d})^\top \bm\alpha = (\bm w\otimes\bm y)^\top \bm\alpha = 0. \label{eq:l2loss-l1reg-constraint-3}
\end{align}
Optimality conditions \eqref{eq:KKT-dual2primal} and \eqref{eq:KKT-primal2dual} are described as
\begin{align}
& \forall j\in[d-1]:~| (\bm w\otimes\check{X}_{:j})^\top \bm\alpha^{*(\bm w)} | < 1 \Rightarrow \beta^{*(\bm w)}_j = 0,
	\label{eq:KKT-dual2primal-l2loss-l1reg-svm}\\
& \forall i\in[n]:\quad \alpha^{*(\bm w)}_i = \frac{2}{\lambda}\max\{0, 1 - \check{X}_{i:}\bm\beta^{*(\bm w)}\}.
	\label{eq:KKT-primal2dual-l2loss-l1reg-svm}
\end{align}

Noticing that ${\cal Z}[\sigma_j^*] = (-\lambda, \lambda)$,
by Theorem \ref{thm:safe-screening-robust},
the DRSfS rule for ${\cal W}$
is calculated as:
\begin{align*}
& \underline{L} - N \overline{R} > -\lambda,
	\quad
	\overline{L} + N \overline{R} < \lambda,
\end{align*}
where
\begin{align*}
& \underline{L} := \lambda \min_{\bm w\in{\cal W}} (\check{X}_{:j} \otimes \bm\alpha^{*(\tilde{\bm w})})^\top \bm w, \\
& \overline{L} := \lambda \max_{\bm w\in{\cal W}} (\check{X}_{:j} \otimes \bm\alpha^{*(\tilde{\bm w})})^\top \bm w, \\
& N := \lambda \sqrt{\max_{\bm w\in{\cal W}} \| \bm w\otimes\check{X}_{:j}\|_2^2 } \\
	& \phantom{N}
	= \lambda \sqrt{\max_{\bm w\in{\cal W}}\{ \bm w^\top \mathrm{diag}(\check{X}_{:j}\otimes\check{X}_{:j}) \bm w \}}, \\
& \overline{R} := \max_{\bm w\in{\cal W}}\bar{r}(\bm w, \bm\gamma, \mu, \bm\beta^{*(\tilde{\bm w})}, \bm\alpha^{*(\tilde{\bm w})}), \\
& \bar{r}(\bm w, \bm\gamma, \mu, \bm\beta^{*(\tilde{\bm w})}, \bm\alpha^{*(\tilde{\bm w})}) \\
	& := \sqrt{\frac{2 \mu}{\min_{i\in[n]} w_i \gamma_i^2}[P_{\bm w}(\bm\beta^{*(\tilde{\bm w})}) - D_{\bm w}(\bm\alpha^{*(\tilde{\bm w})})]},
\end{align*}
\begin{align*}
& P_{\bm w}(\bm\beta^{*(\tilde{\bm w})}) - D_{\bm w}(\bm\alpha^{*(\tilde{\bm w})}) \\
	& = \sum_{i=1}^n w_i \left[ \ell_{y_i}(\check{X}_{i:}\bm\beta^{*(\tilde{\bm w})}) + \lambda \frac{\lambda(\alpha^{*(\tilde{\bm w})})^2_i - 4\alpha^{*(\tilde{\bm w})}_i}{4} \right] \\
	& \phantom{=} + \rho(\bm\beta^{*(\tilde{\bm w})}).
\end{align*}
Here, the expressions in $\underline{L}$ and $\overline{L}$ are linear with respect to $\bm w$, and the expression in $N$ inside the square root is convex and quadratic with respect to $\bm w$.
Also, $\overline{R}$ 
is decomposed to two maximizations $\frac{2 \mu}{\min_{i\in[n]} w_i \gamma_i^2}$ and $P_{\bm w}(\bm\beta^{*(\bm w)}) - D_{\bm w}(\bm\alpha^{*(\bm w)})$, where the former is easily computed while the latter is linear with respect to $\bm w$.
So, similar to L1-loss L2-regularized SVM, we can obtain the maximization result 
by maximizing or minimizing the linear terms by Lemma \ref{lem:optimize-linear} in Appendix \ref{app:proofs},
and maximizing the convex quadratic function by the method of Section \ref{sec:maximize-convex-quadratic}.

\subsection{Maximizing Linear and Convex Quadratic Functions in Hyperball Constraint} \label{sec:maximize-convex-quadratic}

To derive DRSS rules of Sections \ref{sec:l1loss-l2reg-svm} and \ref{sec:l2loss-l1reg-svm}, we need to compute the following forms of optimization problems:
\begin{align}
& \max_{\bm w\in{\cal W}} \bm w^\top A \bm w + 2\bm b^\top\bm w, \label{eq:maximize-convex-quadratic} \\
& \text{where}
	\quad {\cal W} := \{ \bm w\in\mathbb{R}^n \mid \|\bm w - \tilde{\bm w}\|_2\leq S \}, \nonumber\\
& \phantom{\text{where}}
	\quad \tilde{\bm w}\in\mathbb{R}^n,
	\quad \bm b\in\mathbb{R}^n, \nonumber\\
& \phantom{\text{where}}
	\quad A\in\mathbb{R}^{n\times n}:~\text{symmetric, positive semidefinite,} \nonumber \\
& \phantom{\text{where} \quad A\in\mathbb{R}^{n\times n}:~}
	\text{nonzero}. \nonumber
\end{align}

\begin{lemma} \label{lem:maximize-convex-quadratic}
The maximization \eqref{eq:maximize-convex-quadratic} is achieved by the following procedure.
First, we define 
$Q\in\mathbb{R}^{n\times n}$ and $\Phi := \mathrm{diag}(\phi_1, \phi_2, \dots, \phi_n)$
as the eigendecomposition of $A$ such that $A = Q^\top\Phi Q$,
$Q$ is orthogonal ($QQ^\top = Q^\top Q = I$).
Also, let $\bm\xi := -\Phi Q\tilde{\bm w} - Q\bm b \in\mathbb{R}^n$, and 
\begin{align}
& {\cal T}(\nu) = \sum_{i=1}^n \left(\frac{\xi_i}{\nu - \phi_i}\right)^2. \label{eq:lagrangian-result-nu}
\end{align}
Then, the maximization \eqref{eq:maximize-convex-quadratic} is equal to the largest value among them:
\begin{itemize}
\item For each $\nu$ such that
	${\cal T}(\nu) = S^2$ (see Lemma \ref{lem:find-invsq}),
	the value $\nu S^2 + (\nu\tilde{\bm w} + \bm b)^\top Q^\top(\Phi - \nu I)^{-1}\bm\xi + \bm b^\top\tilde{\bm w}$, and
\item For each $\nu\in\{\phi_1, \phi_2, \dots, \phi_n\}$ (duplication removed)
	such that ``$\forall i\in[n]:~\phi_i = \nu \Rightarrow \xi_i = 0$'',
	the value
	\begin{align*}
	& \max_{\bm\tau\in\mathbb{R}^n} [\nu S^2 + (\nu\tilde{\bm w} + \bm b)^\top Q^\top \bm\tau + \bm b^\top\tilde{\bm w}],\\
	& \text{subject to}
		\quad\forall i\in{\cal F}_\nu:\quad \tau_i = \frac{\xi_i}{\phi_i - \nu}, \nonumber\\
	& \phantom{\text{subject to}}
		\quad \sum_{i\in{\cal U}_\nu} \tau_i^2 = S^2 - \sum_{i\in{\cal F}_\nu} \tau_i^2, \nonumber \\
	& \text{where}
		\quad {\cal U}_\nu := \{ i \mid i\in[n],~\phi_i = \nu \},
		\quad {\cal F}_\nu := [n]\setminus{\cal U}_\nu.
	\end{align*}
	(Note that the maximization is easily computed by Lemma \ref{lem:optimize-linear}.)
\end{itemize}

\end{lemma}

The proof is presented in Appendix \ref{app:maximize-convex-quadratic}.

\begin{lemma} \label{lem:find-invsq}
Under the same definitions as Lemma \ref{lem:maximize-convex-quadratic},
The equation ${\cal T}(\nu) = S^2$ can be solved by the following procedure:
Let $\bm e := [e_1, e_2, \dots, e_N]$ ($N\leq n$, $k\neq k^\prime \Rightarrow e_k\neq e_{k^\prime}$) be a sequence of indices such that
\begin{enumerate}
\item $e_k\in[n]$ for any $k\in[N]$,
\item $i\in[n]$ is included in $\bm e$ if and only if $\xi_i\neq 0$, and
\item $\phi_{e_1} \leq \phi_{e_2} \leq \dots \leq \phi_{e_N}$.
\end{enumerate}
Note that, if $\phi_{e_k} < \phi_{e_{k+1}}$ ($k\in[N-1]$), then ${\cal T}(\nu)$ is a convex function
in the interval $(\phi_{e_k}, \phi_{e_{k+1}})$ with $\lim_{\nu\to\phi_{e_k}+0} = \lim_{\nu\to\phi_{e_{k+1}}-0} = +\infty$.
Then, unless $N = 0$, each of the following intervals contains just one solution of ${\cal T}(\nu) = S^2$:
\begin{itemize}
\item Intervals $(-\infty, \phi_{e_1})$ and $(\phi_{e_N}, +\infty)$.
\item Let $\nu^{\#(k)} := \targmin_{\phi_{e_k}<\nu<\phi_{e_{k+1}}} {\cal T}(\nu)$.
	For each $k\in[N-1]$ such that $\phi_{e_k}<\phi_{e_{k+1}}$,
	\begin{itemize}
	\item intervals $(\phi_{e_k}, \nu^{\#(k)})$ and $(\nu^{\#(k)}, \phi_{e_{k+1}})$ if ${\cal T}(\nu^{\#(k)}) < S^2$,
	\item interval $[\nu^{\#(k)}, \nu^{\#(k)}]$ (i.e., point) if ${\cal T}(\nu^{\#(k)}) = S^2$.
	\end{itemize}
\end{itemize}
It follows that ${\cal T}(\nu) = S^2$ has at most $2n$ solutions.
\end{lemma}
By Lemma \ref{lem:find-invsq}, in order to compute the solution of ${\cal T}(\nu) = S^2$,
we have only to compute $\nu^{\#(k)}$ by Newton method or the like,
and to compute the solution for each interval by Newton method or the like.
We show an example of ${\cal T}(\nu)$ in Figure \ref{fig:sum_equals_R2},
and the proof in Appendix \ref{app:find-invsq}.

\begin{figure}[t]
\includegraphics[width=\hsize]{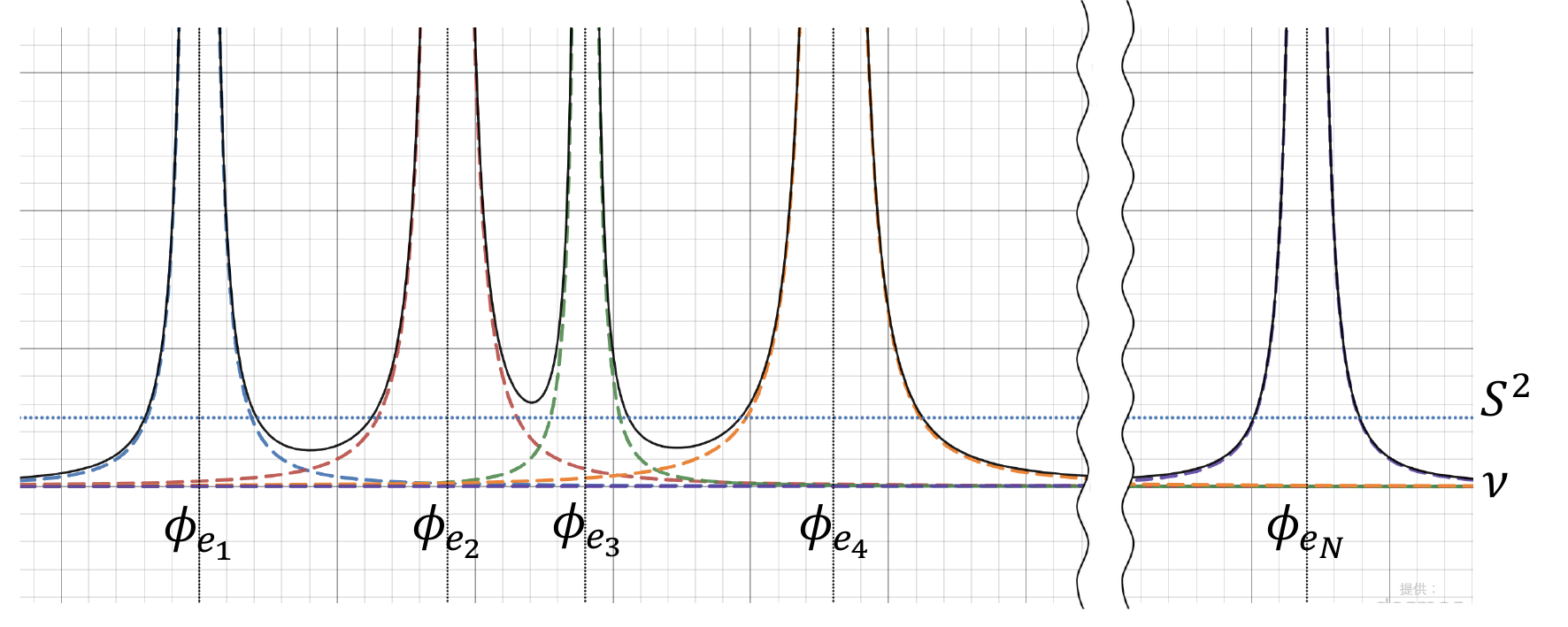}
\vspace{-2em}
\caption{An example of the expression ${\cal T}(\nu)$ (black solid line) in Lemmas \ref{lem:maximize-convex-quadratic} and \ref{lem:find-invsq}.
Colored dash lines denote terms in the summation $(\xi_{e_k}/(\nu - \phi_{e_k}))^2$.
We can see that, given an interval $(\phi_{e_k}, \phi_{e_{k+1}})$ ($k\in[N-1]$), the function is convex.}
\label{fig:sum_equals_R2}
\end{figure}

\section{Application to Deep Learning} \label{sec:deep-learning}

So far, our discussion of SS rules has primarily focused on ML models with linear predictions and convex loss and regularization functions.
However, there may be scenarios where we would like to employ more complex ML models, such as deep learning (DL).

For DL models, deriving SS rules for the entire model can be challenging due to the complexity of bounding the change in model parameters against changes in sample weights.
However, we can simplify the process by focusing on the fact that each layer of DL is often represented as a convex function.
Therefore, we propose applying SS rules specifically to the last layer of DL models.

In this formulation, the layers preceding the last one are considered as
a fixed feature extraction process, even when the sample weights change
(see Figure \ref{fig:deep-learning-with-SS}).
We believe that this approach is valid when the change in sample weights is not significant.
We plan to experimentally evaluate the effectiveness of this formulation in Section \ref{sec:experiment-deep-learning}.

\begin{figure}[t]
\begin{center}
\includegraphics[width=0.6\hsize]{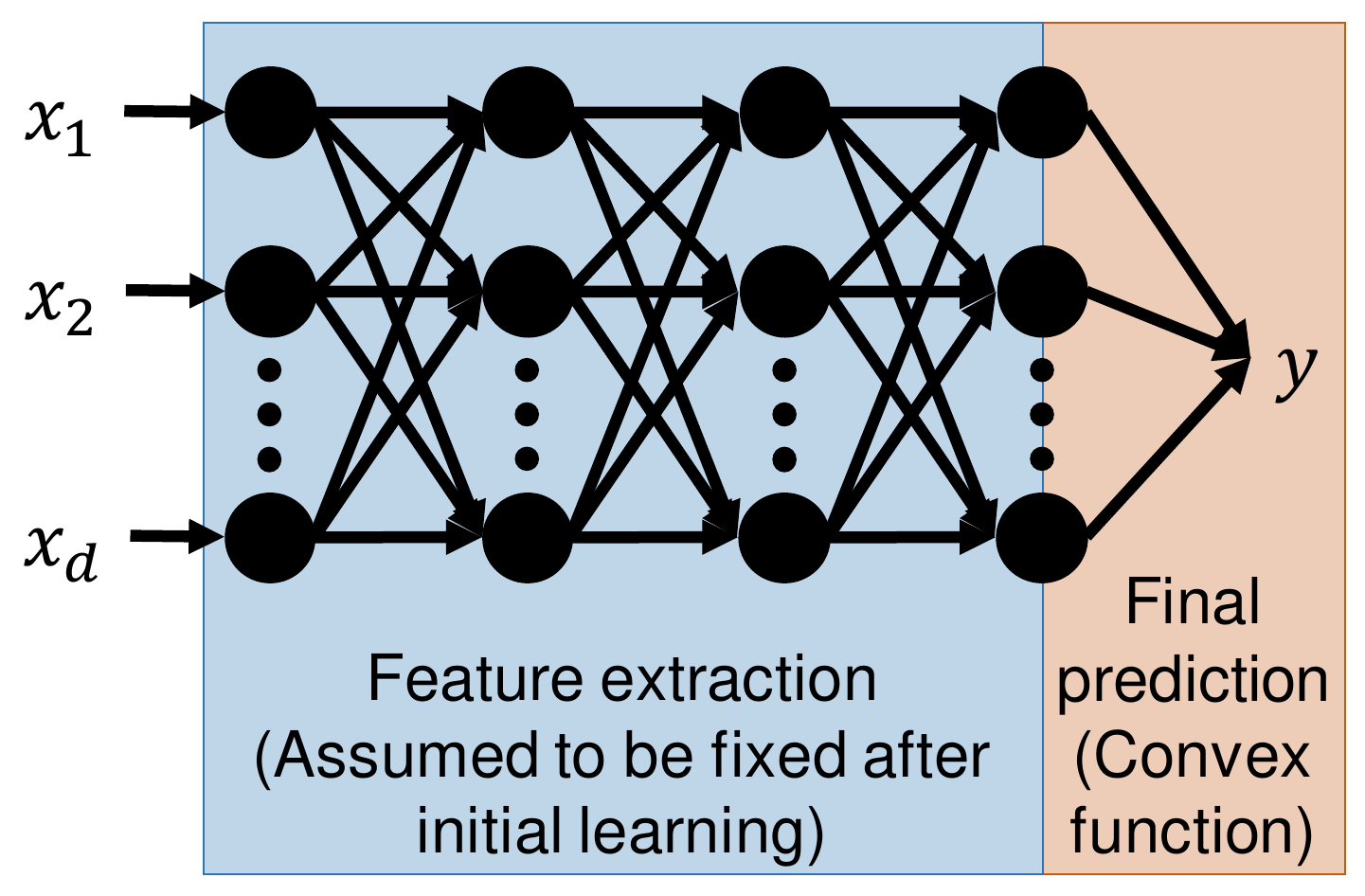}
\end{center}
\vspace{-1em}
\caption{Concept of how to apply SS for deep learning.
SS is applied to the last layer for the final prediction.}
\label{fig:deep-learning-with-SS}
\end{figure}

\section{Numerical Experiment} \label{sec:experiment}

\subsection{Experimental Settings} \label{sec:experimental-settings}

We evaluate the performances of DRSsS and DRSfS across different values of acceptable weight changes $S$ and
hyperparameters for regularization strength $\lambda$.
Performance is measured using \emph{safe screening rates}, representing the ratio of screened samples or features
to all samples or features.
We consider three setups:
DRSsS with L1-loss L2-regularized SVM (Section \ref{sec:l1loss-l2reg-svm}),
DRSfS with L2-loss L1-regularized SVM (Section \ref{sec:l2loss-l1reg-svm}),
and DRSsS with deep learning (Section \ref{sec:deep-learning}) where the last layer incorporates DRSsS with L1-loss L2-regularized SVM.

In these experiments, we set initialize the sample weights before change ($\tilde{\bm w}$) as $\tilde{\bm w} = \bm 1_n$.
Then, we set $S$ in DRSS for ${\cal W} := \{ \bm w \mid \|\bm w - \tilde{\bm w}\|_2\leq S \}$
(Section \ref{sec:DRSS-examples}) as follows:
\begin{itemize}
\item First we assume the weight change that
	the weights for positive samples ($\{i\mid y_i = +1\}$) from $1$ to $a$,
	while retaining the weights for negative samples ($\{i\mid y_i = -1\}$) as $1$.
\item Then, we defined $S$ as the size of weight change above; specifically, we set $S = \sqrt{n^+}|a - 1|$ ($n^+$: number of positive samples in the training dataset).
\end{itemize}
We vary $a$ within the range $0.9\leq a\leq 1.1$, assuming a maximum change of up to 10\% per sample weight.

\subsection{Relationship between the Weight Changes and Safe Screening Rate} \label{sec:experiment-weight-changes}

\begin{table}[t]
\caption{Datasets for DRSsS/DRSfS experiments.
All are binary classification datasets from LIBSVM dataset \cite{libsvmDataset}.
The mark $\dagger$ denotes datasets with one feature removed due to computational constraints.
See Appendix \ref{app:experimental-setup} for details.}
\label{tab:dataset-SS}
\begin{center}
\begin{tabular}{cl|rrr}
\hline
Task & \multicolumn{1}{c|}{Name} & \multicolumn{1}{c}{$n$} & \multicolumn{1}{c}{$n^+$} & \multicolumn{1}{c}{$d$} \\
\hline
DRSsS
& australian        &  690 &  307 & 15 \\
& breast-cancer     &  683 &  239 & 11 \\
& heart             &  270 &  120 & 14 \\
& ionosphere        &  351 &  225 & 35 \\
& sonar             &  208 &   97 & 61 \\
& splice (train)    & 1000 &  517 & 61 \\
& svmguide1 (train) & 3089 & 2000 & 5 \\
\hline
DRSsS
& madelon (train) & 2000 & 1000 & $\dagger$~500 \\
& sonar           &  208 &   97 & $\dagger$~60 \\
& splice (train)  & 1000 &  517 & 61 \\
\hline
\end{tabular}
\end{center}
\end{table}

First, we present safe screening rates for two SVM setups.
The datasets used in these experiments are detailed in Table \ref{tab:dataset-SS}.
In this experiment, we adapt the regularization hyperparameter $\lambda$ based on the characteristics of the data.
These details are described in Appendix \ref{app:experimental-setup}.

As an example, for the ``sonar'' dataset, we show the DRSsS result in Figure \ref{fig:SsS-example}
and the DRSfS result in Figure \ref{fig:SfS-example}.
Results for other datasets are presented in Appendix \ref{app:experiment}.

These plots allow us to assess the tolerance for changes in sample weights.
For instance,
with $a=0.98$ (weight of each positive sample is reduced by two percent, or equivalent weight change in L2-norm),
the sample screening rate is 0.31 for L1-loss L2-regularized SVM with $\lambda=\mathrm{6.58e+1}$, and
the feature screening rate is 0.29 for L2-loss L1-regularized SVM with $\lambda=\mathrm{3.47e+1}$.
This implies that, even if the weights are changed in such ranges,
a number of samples or features are still identified as redundant in the sense of prediction.

\subsection{Safe Sample Screening for Deep Learning Model} \label{sec:experiment-deep-learning}

\begin{figure}[t]
\includegraphics[width=\hsize]{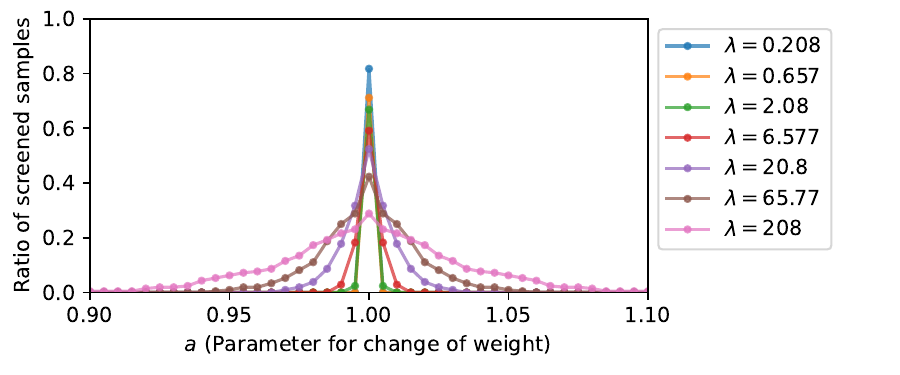}
\vspace{-2.5em}
\caption{Ratio of screened samples by DRSsS for dataset ``sonar''.}
\label{fig:SsS-example}

~

\includegraphics[width=\hsize]{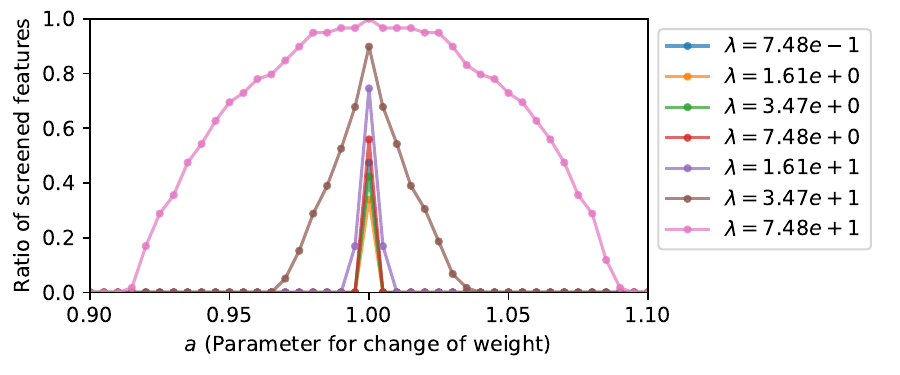}
\vspace{-2.5em}
\caption{Ratio of screened features by DRSfS for dataset ``sonar''.}
\label{fig:SfS-example}

~

\includegraphics[width=\hsize]{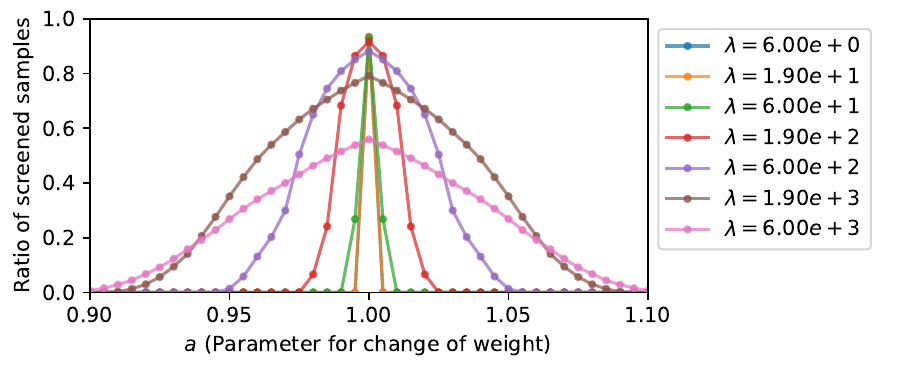}
\vspace{-2em}
\caption{Ratio of screened samples by DRSsS for dataset with CIFAR-10 dataset and DL model ResNet50.}
\label{fig:SsS-DL}
\end{figure}

We applied DRSsS to DL models (Section \ref{sec:deep-learning}), assuming that all layers are fixed except for the last layer.

We utilized a neural network architecture comprising the following components: firstly, \emph{ResNet50} \cite{he2016deep} with an output of 2,048 features, followed by a fully connected layer to reduce the features to 10, and finally, L1-loss L2-regularized SVM (Section \ref{sec:l1loss-l2reg-svm}) accompanied by the intercept feature (Remark \ref{rem:intercept}).

For the experiment, we employed the CIFAR-10 dataset \cite{cifar10}, a well-known benchmark dataset for image classification tasks.
We configured the network to classify images into two classes: ``airplane'' and ``automobile''.
Given that there are 5,000 images for each class,
we split the dataset into training:validation:testing=6:2:2, resulting in a total of 6,000 images in the training dataset.

The resulting safe sample screening rates are illustrated in Figure \ref{fig:SsS-DL}.
We observed similar outcomes to those obtained with ordinary SVMs in Section \ref{sec:experiment-weight-changes}.

This experiment validates the feasibility of applying DRSsS to DL models, demonstrating consistent results with traditional SVM setups.

\section{Conclusion} \label{sec:conclusion}

In this paper, we discussed DR-SS, considering the possible changes in sample weights to represent DR setup.
We developed a method for calculating SS that can handle changes in sample weights by introducing nontrivial computational techniques, such as constrained maximization of certain
convex functions (Section \ref{sec:maximize-convex-quadratic}).
Additionally, to address the constraint of SS, which typically applies to ML by minimizing convex functions, we provided an application to DL by applying SS to the last layer of DL model.
While this approach is an approximation, it holds certain validity.

For the future work, we aim to explore different environmental changes.
In this paper, we focused on weight constraint by L2-norm $\|\bm w - \tilde{\bm w}\|_2 \leq S$ (Section \ref{sec:DRSS-examples})
due to computational considerations.
However, when interpreting changes in weights,
the constraint of L1-norm $\|\bm w - \tilde{\bm w}\|_1 \leq S$ may be more appropriate, as it reflects changes in weights by altering the number of samples.
Furthermore, in the context of DR-SS for DL, we are interested in loosening the constraint of fixing the network except for the last layer.
Investigating this aspect could provide valuable insights into the flexibility of DR-SS methodologies in DL applications.

\section*{Software and Data}

The code and the data to reproduce the experiments are available as the attached file.

\section*{Potential Broader Impact}

This paper contributes to machine learning in dynamically changing environments,
a scenario increasingly prevalent in real-world data analyses.
We believe that, in such situations, ensuring prediction performance against environmental
changes and minimizing storage requirements for expanding datasets will be beneficial.
The method does not present significant ethical concerns or foreseeable societal consequences because this work is theoretical and, as of now, has no direct applications that might impact society or ethical considerations.

\section*{Acknowledgements}

This work was partially supported by MEXT KAKENHI (20H00601), JST CREST (JPMJCR21D3 including AIP challenge program, JPMJCR22N2), JST Moonshot R\&D (JPMJMS2033-05), JST AIP Acceleration Research (JPMJCR21U2), NEDO (JPNP18002, JPNP20006) and RIKEN Center for Advanced Intelligence Project.


\bibliography{ref-SS}
\bibliographystyle{unsrt}

\newpage
\appendix
\onecolumn

\section{Proofs} \label{app:proofs}

\subsection{General Lemmas}

\begin{lemma}\label{lem:fenchel-moreau}
For a convex function $f: \mathbb{R}^d\to\mathbb{R}\cup\{+\infty\}$,
$f^{**}$ is equivalent to $f$ if $f$ is convex, proper (i.e., $\exists \bm v\in\mathbb{R}^d:~f(\bm v) < +\infty$) and lower-semicontinuous.
\end{lemma}

\begin{proof}
See Section 12 of \cite{rockafellar1970convex} for example.
\end{proof}

Lemma \ref{lem:fenchel-moreau} is known as \emph{Fenchel-Moreau theorem}.
Especially, Lemma \ref{lem:fenchel-moreau} holds if $f$ is convex and $\forall \bm v\in\mathbb{R}^d:~f(\bm v)<+\infty$.

\begin{lemma}\label{lem:strong-convexity-smoothness}
For a convex function $f: \mathbb{R}^d\to\mathbb{R}\cup\{+\infty\}$,
\begin{itemize}
\item $f^*$ is $(1/\nu)$-strongly convex if $f$ is proper and $\nu$-smooth.
\item $f^*$ is $(1/\kappa)$-smooth if $f$ is proper, lower-semicontinuous and $\kappa$-strongly convex.
\end{itemize}
\end{lemma}

\begin{proof}
See Section X.4.2 of \cite{hiriart1993convex} for example.
\end{proof}

\begin{lemma}\label{lem:strong-convexity-sphere}
Suppose that $f: \mathbb{R}^d\to\mathbb{R}\cup\{+\infty\}$ is a $\kappa$-strongly convex function,
and let $\bm v^* = \targmin_{\bm v\in\mathbb{R}^d} f(\bm v)$ be the minimizer of $f$.
Then, for any $\bm v\in\mathbb{R}^d$, we have
\begin{align*}
\| \bm v - \bm v^* \|_2 \leq \sqrt{\frac{2}{\kappa}[f(\bm v) - f(\bm v^*)]}.
\end{align*}
\end{lemma}

\begin{proof}
See \cite{ndiaye2015gap} for example.
\end{proof}

\begin{lemma} \label{lem:optimize-linear}
For any vector $\bm a, \bm c\in\mathbb{R}^n$ and $S > 0$,
\begin{align*}
& \min_{\bm v\in\mathbb{R}^n:~\|\bm v - \bm c\|_2\leq S} \bm a^\top \bm v = \bm a^\top \bm c - S\|\bm a\|_2,
& \max_{\bm v\in\mathbb{R}^n:~\|\bm v - \bm c\|_2\leq S} \bm a^\top \bm v = \bm a^\top \bm c + S\|\bm a\|_2. \\
\end{align*}
\end{lemma}

\begin{proof}
By Cauchy-Schwarz inequality,
\begin{align*}
& -\|\bm a\|_2 \|\bm v - \bm c\|_2 \leq \bm a^\top (\bm v - \bm c) \leq \|\bm a\|_2 \|\bm v - \bm c\|_2.
\end{align*}
Noticing that the first inequality becomes equality if $\exists\omega>0:~\bm a = -\omega(\bm v - \bm c)$,
while the second inequality becomes equality if $\exists\omega^\prime>0:~\bm a = \omega^\prime(\bm v - \bm c)$.
Moreover, since $\|\bm v - \bm c\|_2\leq S$,
\begin{align*}
& - S \|\bm a\|_2 \leq \bm a^\top (\bm v - \bm c) \leq S \|\bm a\|_2
\end{align*}
also holds, with the equality holds if $\|\bm v - \bm c\|_2 = S$.

On the other hand, if we take $\bm v$ that satisfies both of the equality conditions
of Cauchy-Schwarz inequality above, that is,
\begin{itemize}
\item (for the first inequality being equality) $\bm v = \bm c - (S/\|\bm a\|_2)\bm a$,
\item (for the second inequality being equality) $\bm v = \bm c + (S/\|\bm a\|_2)\bm a$,
\end{itemize}
then the inequalities become equalities. This proves that $- S \|\bm a\|_2$ and $S \|\bm a\|_2$ are surely the minimum and maximum of $\bm a^\top (\bm v - \bm c)$, respectively.
\end{proof}

\subsection{Derivation of Dual Problem by Fenchel's Duality Theorem} \label{app:fenchel}

As the formulation of Fenchel's duality theorem, we follow the one in Section 31 of \cite{rockafellar1970convex}.

\begin{lemma}[A special case of Fenchel's duality theorem: $f, g<+\infty$] \label{lm:Fenchel-duality}
Let $f: \mathbb{R}^n\to\mathbb{R}$ and $g: \mathbb{R}^d\to\mathbb{R}$ be convex functions,
and $A\in\mathbb{R}^{n\times d}$ be a matrix.
Moreover, we define
\begin{align}
& \bm v^* := \min_{\bm v\in\mathbb{R}^d} [f(A\bm v) + g(\bm v)], \label{eq:FD-primal} \\
& \bm u^* := \max_{\bm u\in\mathbb{R}^n} [ -f^*(-\bm u) - g^*(A^\top\bm u)]. \label{eq:FD-dual}
\end{align}
Then Fenchel's duality theorem assures that
\begin{align*}
& f(A\bm v^*) + g(\bm v^*) = -f^*(-\bm u^*) - g^*(A^\top\bm u^*), \\
& -\bm u^* \in \partial f(A \bm v^*), \\
& \bm v^* \in \partial g^*(A^\top \bm u^*).
\end{align*}
\end{lemma}

\begin{proof}[Sketch of the proof]
Introducing a dummy variable $\bm\psi\in\mathbb{R}^n$ and a Lagrange multiplier $\bm u\in\mathbb{R}^n$, we have
\begin{align}
& \min_{\bm v\in\mathbb{R}^d} [f(A\bm v) + g(\bm v)]
	= \max_{\bm u\in\mathbb{R}^n} \min_{\bm v\in\mathbb{R}^d,~\bm\psi\in\mathbb{R}^n} [f(\bm\psi) + g(\bm v) - \bm u^\top(A\bm v - \bm\psi)] \label{eq:Fenchel-Lagrange}\\
& = - \min_{\bm u\in\mathbb{R}^n} \max_{\bm v\in\mathbb{R}^d,~\bm\psi\in\mathbb{R}^n} [-f(\bm\psi) - g(\bm v) + \bm u^\top(A\bm v - \bm\psi)]
	= - \min_{\bm u\in\mathbb{R}^n} \max_{\bm v\in\mathbb{R}^d,~\bm\psi\in\mathbb{R}^n} [\{ (-\bm u)^\top \bm\psi - f(\bm\psi) \} + \{ (A^\top \bm u)^\top \bm v - g(\bm v) \}] \nonumber\\
& = - \min_{\bm u\in\mathbb{R}^n} [f^*(-\bm u) + g^*(A^\top \bm u)]
	= \max_{\bm u\in\mathbb{R}^n} [-f^*(-\bm u) - g^*(A^\top \bm u)]. \label{eq:Fenchel-dual}
\end{align}
Moreover, by the optimality condition of a problem with a Lagrange multiplier \eqref{eq:Fenchel-Lagrange},
the optima of it, denoted by $\bm v^*$, $\bm\psi^*$ and $\bm u^*$, must satisfy
\begin{align*}
A \bm v^* = \bm\psi^*,
\quad
A^\top \bm u^* \in\partial g(\bm v^*),
\quad
-\bm u^* \in \partial f(\bm\psi^*) = \partial f(A \bm v^*).
\end{align*}
On the other hand, introducing a dummy variable $\bm\phi\in\mathbb{R}^d$ and a Lagrange multiplier $\bm v\in\mathbb{R}^d$ for \eqref{eq:Fenchel-dual}, we have
\begin{align}
& \max_{\bm u\in\mathbb{R}^n} [-f^*(-\bm u) - g^*(A^\top \bm u)]
	= \min_{\bm v\in\mathbb{R}^d} \max_{\bm u\in\mathbb{R}^n, \bm\phi\in\mathbb{R}^d} [-f^*(-\bm u) - g^*(\bm\phi) - \bm v^\top(A^\top \bm u - \bm\phi)] \label{eq:Fenchel-dual-Lagrange}\\
& = \min_{\bm v\in\mathbb{R}^d} \max_{\bm u\in\mathbb{R}^n, \bm\phi\in\mathbb{R}^d} [\{(A\bm v)^\top(-\bm u) - f^*(-\bm u)\} + \{\bm v^\top \bm\phi - g^*(\bm\phi)\}] \nonumber \\
& = \min_{\bm v\in\mathbb{R}^d} [f^{**}(A\bm v) + g^{**}(\bm v)]
	= \min_{\bm v\in\mathbb{R}^d} [f(A\bm v) + g(\bm v)]. \quad(\because~\text{Lemma \ref{lem:fenchel-moreau}})\nonumber
\end{align}
Likely above, by the optimality condition of a problem with a Lagrange multiplier \eqref{eq:Fenchel-dual-Lagrange},
the optima of it, denoted by $\bm u^*$, $\bm\phi^*$ and $\bm v^*$, must satisfy
\begin{align*}
A^\top \bm u^* = \bm\phi^*,
\quad
\bm v^* \in \partial g^*(\bm\phi^*) = \partial g^*(A^\top \bm u^*),
\quad
A \bm v^* \in \partial f(-\bm u^*).
\end{align*}
\end{proof}

\begin{lemma}[Dual problem of weighted regularized empirical risk minimization (weighted RERM)] \label{lm:dual-WRERM}
For the minimization problem
\begin{align}
& \bm\beta^{*(\bm w)} := \argmin_{\bm\beta\in\mathbb{R}^d} P_{\bm w}(\bm\beta), \nonumber
	\quad
	\text{where}
	\quad
	P_{\bm w}(\bm\beta) := \sum_{i=1}^n w_i \ell_{y_i}(\check{X}_{i:}\bm\beta) + \rho(\bm\beta), \tag{\eqref{eq:primal} restated}
\end{align}
we define the dual problem as the one obtained by applying Fenchel's duality theorem (Lemma \ref{lm:Fenchel-duality}), which is defined as
\begin{align}
& \bm\alpha^{*(\bm w)} := \argmax_{\bm\alpha\in\mathbb{R}^n} D_{\bm w}(\bm\alpha), \nonumber
	\quad
	\text{where}
	\quad
	D_{\bm w}(\bm\alpha) := -\sum_{i=1}^n w_i \ell^*_{y_i}(-\gamma_i \alpha_i) + \rho^*(((\bm\gamma\otimes\bm w)\sqtimes\check{X})^\top \bm\alpha). \tag{\eqref{eq:dual} restated}
\end{align}
Moreover, $\bm\beta^{*(\bm w)}$ and $\bm\alpha^{*(\bm w)}$ must satisfy
\begin{align}
& P_{\bm w}(\bm\beta^{*(\bm w)}) = D_{\bm w}(\bm\alpha^{*(\bm w)}), \tag{\eqref{eq:strong-duality} restated}\\
& \bm\beta^{*(\bm w)} \in \partial\rho^*(((\bm\gamma\otimes\bm w)\sqtimes\check{X})^\top \bm\alpha^{*(\bm w)}), \tag{\eqref{eq:KKT-dual2primal} restated}\\
& \forall i\in[n]:\quad -\gamma_i \alpha^{*(\bm w)}_i \in \partial\ell_{y_i}(\check{X}_{i:}\bm\beta^{*(\bm w)}). \tag{\eqref{eq:KKT-primal2dual} restated}
\end{align}
\end{lemma}

\begin{proof}
To apply Fenchel's duality theorem, we have only to set $f$, $g$ and $A$ in Lemma \ref{lm:Fenchel-duality} as
\begin{align*}
f(\bm u) := \sum_{i=1}^n w_i \ell_{y_i}(u_i),
\quad
g(\bm\beta) := \rho(\bm\beta),
\quad
A := \check{X}.
\end{align*}
Here, noticing that
\begin{align*}
& f^*(\bm u)
	= \sup_{\bm u^\prime\in\mathbb{R}^n} [\bm u^\top \bm u^\prime - \sum_{i=1}^n w_i \ell_{y_i}(u^\prime_i)]
	= \sup_{\bm u^\prime\in\mathbb{R}^n} \sum_{i=1}^n \left[u_i u^\prime_i - w_i \ell_{y_i}(u^\prime_i) \right] \\
& = \sup_{\bm u^\prime\in\mathbb{R}^n} \sum_{i=1}^n w_i \left[\frac{u_i}{w_i} u^\prime_i - \ell_{y_i}(u^\prime_i) \right]
	= \sum_{i=1}^n w_i \ell_{y_i}^*\left(\frac{u_i}{w_i}\right),
\end{align*}
from \eqref{eq:FD-dual} we have
\begin{align*}
-f^*(-\bm u) - g^*(A^\top\bm u)
= - \sum_{i=1}^n w_i \ell_{y_i}^*\left(-\frac{u_i}{w_i}\right) - \rho^*(\check{X}^\top\bm u).
\end{align*}
Replacing $u_i\gets \gamma_i w_i \alpha_i$, that is, $\bm u\gets (\bm\gamma \otimes \bm w \otimes \bm\alpha)$,
we have the dual problem \eqref{eq:dual}.

The relationships between the primal and the dual problem are described as follows:
\begin{align*}
& -\bm u^* \in \partial f(A \bm v^*)
	~\Rightarrow~
	-\bm\gamma \otimes \bm w \otimes \bm\alpha^{*(\bm w)} \in \partial f(\check{X} \bm\beta^{*(\bm w)})
	~\Rightarrow~
	-\gamma_i w_i \alpha^{*(\bm w)}_i \in w_i \partial\ell_{y_i}(\check{X}_{i:}\bm\beta^{*(\bm w)}) \\
	& \Rightarrow
	-\gamma_i \alpha^{*(\bm w)}_i \in \partial\ell_{y_i}(\check{X}_{i:}\bm\beta^{*(\bm w)}), \\
& \bm v^* \in \partial g^*(A^\top \bm u^*)
	~\Rightarrow~
	\bm\beta^{*(\bm w)} \in \partial g^*(\check{X}^\top \bm\gamma \otimes \bm w \otimes \bm\alpha^{*(\bm w)})
	= \partial g^*(((\bm\gamma \otimes \bm w)\sqtimes\check{X})^\top \bm\alpha^{*(\bm w)}).
\end{align*}
\end{proof}

\subsection{Proof of Lemma \ref{lem:gap-sphere-primal}} \label{app:gap-sphere-primal}

\begin{proof} \cite{ndiaye2015gap}
\begin{align*}
\| \hat{\bm\beta} - \bm\beta^{*(\bm w)} \|_2 
	& \leq \sqrt{\frac{2}{\lambda}[P_{\bm w}(\hat{\bm\beta}) - P_{\bm w}(\bm\beta^{*(\bm w)})]}
	\tag{$\because$ setting $f\gets P_{\bm w}$ in Lemma \ref{lem:strong-convexity-sphere}} \\
	& = \sqrt{\frac{2}{\lambda}[P_{\bm w}(\hat{\bm\beta}) - D_{\bm w}(\bm\alpha^{*(\bm w)})]}
	\tag{$\because$ \eqref{eq:strong-duality}} \\
	& \leq \sqrt{\frac{2}{\lambda}[P_{\bm w}(\hat{\bm\beta}) - D_{\bm w}(\hat{\bm\alpha})]}.
	\tag{$\because$ $\bm\alpha^{*(\bm w)}$ is a maximizer of $D_{\bm w}$}
\end{align*}
\end{proof}

\subsection{Proof of Lemma \ref{lem:gap-sample-screening}} \label{app:gap-sample-screening}

\begin{proof}
Due to \eqref{eq:KKT-primal2dual}, if $\partial\ell_{y_i}(\check{X}_{i:}\bm\beta^{*(\bm w)}) = \{0\}$ is assured, then $\alpha_i^{*(\bm w)} = 0$ is assured.
Since we do not know $\bm\beta^{*(\bm w)}$ but know ${\cal B}^{*(\bm w)}$ (Lemma \ref{lem:gap-sphere-primal}),
we can assure $\alpha_i^{*(\bm w)} = 0$ if $\bigcup_{\bm\beta\in{\cal B}^{*(\bm w)}} \partial\ell_{y_i}(\check{X}_{i:}\bm\beta) = \{0\}$ is assured.
Noticing that $\partial\ell_{y_i}$ is monotonically increasing\footnote{Since $\partial\ell_{y_i}$ is a multi-valued function, the monotonicity must be defined accordingly: we call a multi-valued function $F: \mathbb{R}\to 2^{\mathbb{R}}$ is monotonically increasing if, for any $t < t^\prime$, $F$ must satisfy ``$\forall s\in F(t)$, $\forall s^\prime\in F(t^\prime)$:~$s\leq s^\prime$''.}, we have
\begin{align*}
& \bigcup_{\bm\beta\in{\cal B}^{*(\bm w)}} \partial\ell_{y_i}(\check{X}_{i:}\bm\beta) = \{0\}
	\quad\Leftrightarrow\quad
	\bigcup_{\bm\beta\in{\cal B}^{*(\bm w)}} \check{X}_{i:}\bm\beta \subseteq {\cal Z}[\ell_{y_i}]
	\quad\Leftrightarrow\quad
	[\min_{\bm\beta\in{\cal B}^{*(\bm w)}} \check{X}_{i:}\bm\beta, \max_{\bm\beta\in{\cal B}^{*(\bm w)}} \check{X}_{i:}\bm\beta] \subseteq {\cal Z}[\ell_{y_i}] \\
& \Leftrightarrow\quad \left[\check{X}_{i:}\hat{\bm\beta} - \|\check{X}_{i:}\|_2 r(\bm w, \bm\gamma, \kappa, \hat{\bm\beta}, \hat{\bm\alpha}),~
	\check{X}_{i:}\hat{\bm\beta} + \|\check{X}_{i:}\|_2 r(\bm w, \bm\gamma, \kappa, \hat{\bm\beta}, \hat{\bm\alpha}) \right]
	\subseteq {\cal Z}[\ell_{y_i}].
	\tag{$\because$ Lemma \ref{lem:optimize-linear}}
\end{align*}
\end{proof}

\subsection{Proof of Lemma \ref{lem:gap-sphere-dual}} \label{app:gap-sphere-dual}

\begin{proof}
The proof is almost the same as that for Lemma \ref{lem:gap-sphere-primal} (see Appendix \ref{app:gap-sphere-primal}), but we additionally need to show that $-D_{\bm w}$ is $((\min_{i\in[n]} w_i \gamma_i^2) / \mu)$-strongly convex (in this case $D_{\bm w}$ is called \emph{strongly concave}).

As discussed in Lemma \ref{lem:strong-convexity-smoothness}, $-\ell^*_{y_i}(t)$ is $(1/\mu)$-strongly convex,
that is, $-\ell^*_{y_i}(t) - (1/2\mu)t^2$ is convex. Thus,
\begin{itemize}
\item $-\ell^*_{y_i}(-\gamma_i \alpha_i) - (1/2\mu)(\gamma_i \alpha_i)^2$ is convex with respect to $\alpha_i$,
\item $-w_i \ell^*_{y_i}(-\gamma_i \alpha_i) - (w_i \gamma_i^2 /2\mu) \alpha_i^2$ is convex with respect to $\alpha_i$,
\item $-\sum_{i=1}^n w_i \ell^*_{y_i}(-\gamma_i \alpha_i) - \sum_{i=1}^n (w_i \gamma_i^2 /2\mu) \alpha_i^2$ is convex with respect to $\bm\alpha$.
\end{itemize}
So, $-\sum_{i=1}^n w_i \ell^*_{y_i}(-\gamma_i \alpha_i)$ is convex with respect to $\bm\alpha$ even subtracted by $\sum_{k=1}^n [\min_{i\in[n]} (w_i \gamma_i^2 /2\mu)] \alpha_k^2 = (1/2)[\min_{i\in[n]} (w_i \gamma_i^2 /\mu)] \|\bm\alpha\|_2^2$.
\end{proof}


\subsection{Proof of Lemma \ref{lem:maximize-convex-quadratic}} \label{app:maximize-convex-quadratic}

\begin{lemma}
For the optimization problem
\begin{align}
& \max_{\bm w\in{\cal W}} \bm w^\top A \bm w + 2\bm b^\top\bm w, \tag{\eqref{eq:maximize-convex-quadratic} restated} \\
& \text{subject to}
	\quad {\cal W} := \{ \bm w\in\mathbb{R}^n \mid \|\bm w - \tilde{\bm w}\|_2\leq S \}, \nonumber\\
& \text{where}
	\quad \tilde{\bm w}\in\mathbb{R}^n,
	\quad \bm b\in\mathbb{R}^n, \nonumber\\
& \phantom{\text{where}}
	\quad A\in\mathbb{R}^{n\times n}:~\text{symmetric, positive semidefinite, nonzero,} \nonumber
\end{align}
its stationary points are obtained as the solution of the following equations with respect to $\bm w$ and $\nu\in\mathbb{R}$:
\begin{align}
& A \bm w + \bm b - \nu(\bm w - \tilde{\bm w}) = \bm 0, \label{eq:lagrangian-target} \\
& \|\bm w - \tilde{\bm w}\|_2 = S. \label{eq:lagrangian-sphere}
\end{align}
Also, when both \eqref{eq:lagrangian-target} and \eqref{eq:lagrangian-sphere} are satisfied,
the function to be maximized is calculated as
\begin{align}
\bm w^\top A \bm w + 2\bm b^\top\bm w
	= \nu S^2 + (\nu \tilde{\bm w} + \bm b)^\top (\bm w - \tilde{\bm w}) + \tilde{\bm w}^\top \bm b.
	\label{eq:maximization-replaced}
\end{align}
\end{lemma}

\begin{proof}
First, $\bm w^\top A \bm w + 2\bm b^\top\bm w$ is convex and not constant.
Then we can show that \eqref{eq:maximize-convex-quadratic} is optimized in $\{ \bm w\in\mathbb{R}^n \mid \|\bm w - \tilde{\bm w}\|_2 = S \}$, that is, at the surface of the hyperball ${\cal W}$
(Theorem 32.1 of \cite{rockafellar1970convex}). This proves \eqref{eq:lagrangian-sphere}.
Moreover, with the fact, we write the Lagrangian function with Lagrange multiplier $\nu\in\mathbb{R}$ as:
\begin{align*}
L(\bm w, \nu) := \bm w^\top A \bm w + 2\bm b^\top\bm w - \nu(\|\bm w - \tilde{\bm w}\|_2^2 - S^2).
\end{align*}
Then, due to the property of Lagrange multiplier,
the stationary points of \eqref{eq:maximize-convex-quadratic} are obtained as
\begin{align*}
& \frac{\partial L}{\partial\bm w} = 2 A \bm w + 2\bm b - 2 \nu(\bm w - \tilde{\bm w}) = 0, \\
& \frac{\partial L}{\partial\nu} = \|\bm w - \tilde{\bm w}\|_2^2 - S^2 = 0,
\end{align*}
where the former derives \eqref{eq:lagrangian-target}.

Finally we show \eqref{eq:maximization-replaced}.
If both \eqref{eq:lagrangian-target} and \eqref{eq:lagrangian-sphere} are satisfied,
\begin{align}
\bm w^\top A\bm w + 2 \bm b^\top \bm w
& = \bm w^\top(\nu(\bm w - \tilde{\bm w}) - \bm b) + 2 \bm b^\top \bm w \tag{$\because$~\eqref{eq:lagrangian-target}}\\
& = \nu\bm w^\top(\bm w - \tilde{\bm w}) + \bm b^\top \bm w \nonumber\\
& = \nu(\bm w - \tilde{\bm w})^\top(\bm w - \tilde{\bm w}) + \nu\tilde{\bm w}^\top(\bm w - \tilde{\bm w}) + \bm b^\top(\bm w - \tilde{\bm w}) + \bm b^\top\tilde{\bm w} \nonumber\\
& = \nu S^2 + \nu\tilde{\bm w}^\top(\bm w - \tilde{\bm w}) + \bm b^\top(\bm w - \tilde{\bm w}) + \bm b^\top\tilde{\bm w} \tag{$\because$~\eqref{eq:lagrangian-sphere}}\\
& = \nu S^2 + (\nu\tilde{\bm w} + \bm b)^\top(\bm w - \tilde{\bm w}) + \bm b^\top\tilde{\bm w} \tag{\eqref{eq:maximization-replaced} restated}
\end{align}
\end{proof}

\begin{proof}[Proof of Lemma \ref{lem:maximize-convex-quadratic}]
The condition \eqref{eq:lagrangian-target} is calculated as
\begin{align*}
& A \bm w + \bm b = \nu(\bm w - \tilde{\bm w}), \\
& (A - \nu I)(\bm w - \tilde{\bm w}) = -A\tilde{\bm w} - \bm b.
\end{align*}
Here, let us apply eigendecomposition of $A$, denoted by $A = Q^\top\Phi Q$,
where $Q\in\mathbb{R}^{n\times n}$ is orthogonal ($QQ^\top = Q^\top Q = I$) and
$\Phi := \mathrm{diag}(\phi_1, \phi_2, \dots, \phi_n)$ is a diagonal matrix consisting of eigenvalues of $A$.
Such a decomposition is assured to exist since $A$ is assumed to be symmetric and positive semidefinite.
Then,
\begin{align}
& (Q^\top\Phi Q - \nu I)(\bm w - \tilde{\bm w}) = -Q^\top\Phi Q\tilde{\bm w} - \bm b, \nonumber\\
& Q^\top(\Phi - \nu I) Q (\bm w - \tilde{\bm w}) = -Q^\top\Phi Q\tilde{\bm w} - \bm b, \nonumber\\
& (\Phi - \nu I) \bm\tau = \bm\xi,
	\quad
	(\text{where}
	\quad \bm\tau := Q (\bm w - \tilde{\bm w}),
	\quad \bm\xi := -\Phi Q\tilde{\bm w} - Q\bm b \in\mathbb{R}^n,) \label{eq:lagrangian-result-w} \\
& \forall i\in[n]:\quad (\phi_i - \nu) \tau_i = \xi_i. \label{eq:equation-by-tau}
\end{align}
Note that we have to be also aware of the constraint
\begin{align}
S = \|\bm\tau\|_2 = \sqrt{\bm\tau^\top\bm\tau} = \sqrt{(\bm w - \tilde{\bm w})^\top Q^\top Q (\bm w - \tilde{\bm w})} = \|\bm w - \tilde{\bm w}\|_2. \label{eq:tau-constraint}
\end{align}

Here, we consider these two cases.
\begin{enumerate}
\item First, consider the case when $(\Phi - \nu I)$ is nonsingular, that is, when $\nu$ is different from any of $\phi_1, \phi_2, \dots, \phi_n$. Then, from \eqref{eq:tau-constraint} we have
	\begin{align}
	& S^2
		= \|\bm\tau\|_2 = \sum_{i=1}^n \tau_i^2
		= \sum_{i=1}^n \left(\frac{\xi_i}{\nu - \phi_i}\right)^2 \quad\bigl(=: {\cal T}(\nu)\bigr). \label{eq:lagrangian-result-nu-equation}
	\end{align}
	So, values of \eqref{eq:maximize-convex-quadratic} for all stationary points with respect to $\bm w$ and $\nu$ (on condition that $(\Phi - \nu I)$ is nonsingular) can be obtained by computing \eqref{eq:maximization-replaced} for each $\nu$ satisfying \eqref{eq:lagrangian-result-nu-equation}, that is,
	\begin{itemize}
	\item for such $\nu$ computing $\bm\tau$ by \eqref{eq:equation-by-tau}, and
	\item computing \eqref{eq:maximization-replaced} as
		$\nu S^2 + (\nu\tilde{\bm w} + \bm b)^\top(\bm w - \tilde{\bm w}) + \bm b^\top\tilde{\bm w}
		= \nu S^2 + (\nu\tilde{\bm w} + \bm b)^\top Q^\top \bm\tau + \bm b^\top\tilde{\bm w}$.
	\end{itemize}

\item Secondly, consider the case when $(\Phi - \nu I)$ is nonsingular, that is, when $\nu$ is equal to one of $\phi_1, \phi_2, \dots, \phi_n$.
	First, given $\nu$, let ${\cal U}_\nu := \{ i \mid i\in[n],~\phi_i = \nu \}$ be the indices of $\{\phi_i\}_i$ equal to $\nu$ (this may include more than one indices), and ${\cal F}_\nu := [n]\setminus{\cal U}_\nu$. Note that, by assumption, ${\cal U}_\nu$ is not empty.
	Then, all stationary points of \eqref{eq:maximize-convex-quadratic} with respect to $\bm w$ and $\nu$ (on condition that $(\Phi - \nu I)$ is singular) can be found by computing the followings for each $\nu\in\{\phi_1, \phi_2, \dots, \phi_n\}$ (duplication excluded):
	\begin{itemize}
	\item If $\xi_i\neq 0$ for at least one $i\in{\cal U}_\nu$, the equation \eqref{eq:equation-by-tau} cannot hold.
	\item If $\xi_i = 0$ for all $i\in{\cal N}_\nu$, the equation \eqref{eq:equation-by-tau} may hold.
		So we calculate $\bm\tau$ that maximizes \eqref{eq:maximize-convex-quadratic} as follows:
		\begin{itemize}
		\item Fix $\tau_i = \xi_i / (\phi_i - \nu)$ for $i\in{\cal F}_\nu$.
		\item Set the constraint $\sum_{i\in{\cal U}_\nu} \tau_i^2 = S^2 - \sum_{i\in{\cal F}_\nu} \tau_i^2$ (due to \eqref{eq:tau-constraint}).
		\item Maximize \eqref{eq:maximize-convex-quadratic} with respect to $\{\tau_i\}_{i\in{\cal U}_\nu}$
			under the constraints above. Here, by \eqref{eq:maximization-replaced} we have only to calculate
			\begin{align}
			& \max_{\bm\tau\in\mathbb{R}^n} [\nu S^2 + (\nu\tilde{\bm w} + \bm b)^\top(\bm w - \tilde{\bm w}) + \bm b^\top\tilde{\bm w}], \label{eq:singular-max}\\
			& \text{subject to}
				\quad\forall i\in{\cal F}_\nu:\quad \tau_i = \frac{\xi_i}{\phi_i - \nu}, \nonumber\\
			& \phantom{\text{subject to}}
				\quad \sum_{i\in{\cal U}_\nu} \tau_i^2 = S^2 - \sum_{i\in{\cal F}_\nu} \tau_i^2, \nonumber
			\end{align}
			which is easily computed by Lemma \ref{lem:optimize-linear}.
			The value of the maximization result is equal to that of \eqref{eq:maximize-convex-quadratic} on condition that $\nu$ is specified above.
		\end{itemize}
		So, collecting these result and taking the largest one, the maximization (on condition that $(\Phi - \nu I)$ is singular) is completed.
	\end{itemize}
\end{enumerate}
Taking the maximum of the two cases, we have the maximization result of \eqref{eq:maximize-convex-quadratic}.
\end{proof}

\subsection{Proof of Lemma \ref{lem:find-invsq}} \label{app:find-invsq}

\begin{proof}
We show the statements in the lemma that, if $\phi_{e_k} < \phi_{e_{k+1}}$ ($k\in[N-1]$), then ${\cal T}(\nu)$ is a convex function
in the interval $(\phi_{e_k}, \phi_{e_{k+1}})$ with $\lim_{\nu\to\phi_{e_k}+0} = \lim_{\nu\to\phi_{e_{k+1}}-0} = +\infty$. Then the conclusion immediately follows.

The latter statement clearly holds. The former statement is proved by directly computing the derivative.
\begin{align*}
& \frac{d}{d\nu}{\cal T}(\nu)
	= \frac{d}{d\nu}\sum_{i=1}^n \left(\frac{\xi_i}{\nu - \phi_i}\right)^2
	= -2\sum_{i=1}^n \frac{\xi_i^2}{(\nu - \phi_i)^3}.
\end{align*}
It is an increasing function with respect to $\nu$, as long as $\nu$ does not match any of $\{\phi_i\}_{i=1}^n$
such that $\xi_i\neq 0$. So it is convex in the interval $\phi_{e_k} < \nu < \phi_{e_{k+1}}$.
\end{proof}

\section{Detailed Calculations} \label{app:details-for-problems}

In this appendix we describe detailed calculations omitted in the main paper.

\subsection{Calculations for L1-loss L2-regularized SVM (Section \ref{sec:l1loss-l2reg-svm})} \label{app:details-l1loss-l2reg-svm}

For this setup, we can calculate as
\begin{align*}
& \rho^*(\bm\beta) := \frac{1}{2\lambda}\|\bm\beta\|_2^2,
\quad
	\ell^*_y(t) := \begin{cases}
	t, & (-1 \leq t \leq 0) \\
	+\infty, & (\text{otherwise})
	\end{cases}
\quad
	\partial\rho^*(\bm\beta) := \left\{ \frac{1}{\lambda}\bm\beta \right\},
\quad
	\partial\ell_y(t) := \begin{cases}
	\{ -1 \}, & (t < 1) \\
	[-1, 0], & (t = 1) \\
	\{ 0 \}. & (t > 1)
	\end{cases}
\end{align*}
Then we have the dual problem in the main paper \eqref{eq:dual-l1loss-l2reg-svm}.

\subsection{Calculations for L2-loss L1-regularized SVM (Section \ref{sec:l2loss-l1reg-svm})} \label{app:details-l2loss-l1reg-svm}

For this setup, we can calculate as
\begin{align*}
& \rho^*(\bm\beta) := \begin{cases}
	0, & (\beta_d=0,~\forall j\in[d-1]:~|\beta_j|\leq\lambda) \\
	+\infty, & (\text{otherwise})
	\end{cases}
\quad
	\ell^*_y(t) := \begin{cases}
	\frac{t^2+4t}{4}, & (t \leq 0) \\
	+\infty, & (\text{otherwise})
	\end{cases}
	\\
& \forall j\in[d-1]:~[\partial\rho^*(\bm\beta)]_j := \begin{cases}
	-\infty, & (\beta_j < -\lambda) \\
	[-\infty, 0], & (\beta_j = -\lambda) \\
	0, & (|\beta_j|<\lambda) \\
	[0, +\infty], & (\beta_j = \lambda) \\
	+\infty, & (\beta_j > \lambda)
	\end{cases}
\quad
	[\partial\rho^*(\bm\beta)]_d := \begin{cases}
	-\infty, & (\beta_d < 0) \\
	[-\infty, +\infty], & (\beta_d = 0) \\
	+\infty, & (\beta_d > 0)
	\end{cases}
	\\
& \partial\ell_y(t) := -2\max\{0, 1 - t\}.
\end{align*}

Then, setting $\gamma_i = \lambda$ for all $i\in[n]$, the dual objective function is described as
\begin{align}
& D_{\bm w}(\bm\alpha) = \begin{cases}
	-\sum_{i=1}^n w_i \frac{\lambda^2 \alpha^2_i - 4\lambda\alpha_i}{4}, & (\text{if~\eqref{eq:l2loss-l1reg-constraint-base}~are~satisfied}) \\
	+\infty, & (\text{otherwise})
	\end{cases}
	\label{eq:dual-l2loss-l1reg-svm-base}
\end{align}
where
\begin{subequations}
\label{eq:l2loss-l1reg-constraint-base}
\begin{align}
& \lambda\alpha_i \geq 0
	\Leftrightarrow \alpha_i \geq 0, \\
& \forall j\in[d-1]:~
	| ((\lambda\bm 1_n \otimes \bm w)\otimes\check{X}_{:j})^\top \bm\alpha | \leq \lambda
	\Leftrightarrow | (\bm w\otimes\check{X}_{:j})^\top \bm\alpha | \leq 1, \\
& ((\lambda\bm 1_n \otimes \bm w)\otimes\check{X}_{:d})^\top \bm\alpha = 0
	\Leftrightarrow (\bm w\otimes\check{X}_{:d})^\top \bm\alpha = 0.
\end{align}
\end{subequations}
Optimality conditions \eqref{eq:KKT-dual2primal} and \eqref{eq:KKT-primal2dual} are described as
\begin{align}
& \forall j\in[d-1]:~
	| (\lambda\bm 1_n\otimes\bm w \otimes\check{X}_{:j})^\top \bm\alpha^{*(\bm w)} | < \lambda
	\Leftrightarrow | (\bm w\otimes\check{X}_{:j})^\top \bm\alpha^{*(\bm w)} | < 1
	\Rightarrow \beta^{*(\bm w)}_j = 0, \\
& \forall i\in[n]:\quad \lambda\alpha^{*(\bm w)}_i = 2\max\{0, 1 - \check{X}_{i:}\bm\beta^{*(\bm w)}\}.
\end{align}

\section{Application of Safe Sample Screening to Kernelized Features} \label{app:kernelized}

The kernel method in ML means computation methods when the input variable vector of a sample $\bm x\in\mathbb{R}^d$ cannot be specifically obtained
(this includes the case when $d$ is infinite),
but for the input variable vectors for any two samples $\bm x, \bm x^\prime\in\mathbb{R}^d$ its inner product $\bm x^\top \bm x^\prime$ can be obtained.
In such a case, we cannot discuss SfS since we cannot obtain each feature specifically,
however, we can discuss SsS.

We show that the SsS rules for L1-loss L2-regularized SVM (Section \ref{sec:l1loss-l2reg-svm})
can be applied even if the features are kernelized.

First, if features are kernelized, we cannot obtain either $X$ or $\bm\beta^{*(\tilde{\bm w})}$ specifically.
However, since we can obtain $\bm\alpha^{*(\tilde{\bm w})}$,
with \eqref{eq:KKT-dual2primal-l1loss-l2reg-svm} we have
\begin{align}
\forall \bm x\in\mathbb{R}^d:~
\bm x^\top \bm\beta^{*(\tilde{\bm w})} = \frac{1}{\lambda}\bm x^\top (\bm w\sqtimes\check{X})^\top \bm\alpha^{*(\tilde{\bm w})}
= \frac{1}{\lambda}\sum_{i=1}^n w_i \alpha_i^{*(\tilde{\bm w})} (\bm x^\top \check{X}_{i:}).
\label{eq:kernelized-inner-product}
\end{align}
This means that we can calculate the inner product of $\bm\beta^{*(\tilde{\bm w})}$ and any vector.

Then, in order to calculate the quantity \eqref{eq:safe-sample-screening-l1loss-l2reg-svm} to conduct SsS,
we have only to calculate
\begin{itemize}
\item $\check{X}_{i:}\bm\beta^{*(\tilde{\bm w})}$ can be calculated by \eqref{eq:kernelized-inner-product},
\item $\|\check{X}_{i:}\|_2 = \sqrt{\check{X}_{i:}^\top \check{X}_{i:}}$ is obtained as the kernel value, and
\item $P_{\bm w}(\bm\beta^{*(\tilde{\bm w})}) - D_{\bm w}(\bm\alpha^{*(\tilde{\bm w})})$ can be calculated by \eqref{eq:kernelized-inner-product} and kernel values since two variables whose values cannot be specifically obtained ($\tilde{X}$ and $\bm\beta^{*(\tilde{\bm w})}$) appears only as inner products.
\end{itemize}

So, all values needed to derive SsS rules \eqref{eq:safe-sample-screening-l1loss-l2reg-svm}
can be computed even if features are kernelized.

\section{Details of Experiments}

\subsection{Detailed Experimental Setup} \label{app:experimental-setup}

The criteria of selecting datasets (Table \ref{tab:dataset-SS}) and detailed setups are as follows:
\begin{itemize}
\item All of the datasets are downloaded from LIBSVM dataset \cite{libsvmDataset}.
	We used scaled datasets for ones used in DRSfS or only scaled datasets are provided
	(``ionosphere'', ``sonar'' and ``splice'').
	We used training datasets only if test datasets are provided separately
	(``splice'', ``svmguide1'' and ``madelon'').
\item For DRSsS,
	we selected datasets from LIBSVM dataset containing 100 to 10,000 samples,
	100 or fewer features, and the area under the curve (AUC) of
	the receiver operating characteristic (ROC) is 0.9 or higher
	for the regularization strengths ($\lambda$) we examined
	so that they tend to facilitate more effective sample screening.
\item For DRSfS,
	we selected datasets from LIBSVM dataset containing 50 to 1,000 features,
	10,000 or fewer samples,
	and containing no categorical features.
	Also, due to computational constraints, we excluded features
	that have at least one zero (marked ``$\dagger$'' in Table \ref{tab:dataset-SS}).
	As a result, one feature from ``madelon'' and one from ``sonar'' have been excluded.
\item In the table, the column ``$d$'' denotes the number of features including the intercept feature (Remark \ref{rem:intercept}).
\end{itemize}

The choice of regularization hyperparameter $\lambda$, based on the characteristics of the data, is as follows:
\begin{itemize}
\item For DRSsS, we set $\lambda$ as $n$, $n\times 10^{-0.5}$, $n\times 10^{-1.0}$, $\ldots$, $n\times 10^{-3.0}$.
	(For DRSsS with DL, we set 1000 instead of $n$.)
	This is because the effect of $\lambda$ gets weaker for larger $n$.
\item For DRSfS, we determine $\lambda$ based on $\lammax$, defined as the smallest $\lambda$ for which $\beta^{*(\bm w)}_j = 0$ for any $j\in[d-1]$ explained below. We then set $\lambda$ as $\lammax$, $\lammax\times 10^{-1/3}$, $\lammax\times 10^{-2/3}$, $\ldots$, $\lammax\times 10^{-2}$.
\end{itemize}

Finally, we show the calculation of $\lammax$ for L2-loss L1-regularized SVM.
By \eqref{eq:KKT-dual2primal-l2loss-l1reg-svm},
we would like to find $\lambda$ so that $| (\bm w\otimes\check{X}_{:j})^\top \bm\alpha^{*(\bm w)} | < 1$
for all $j\in[d-1]$.
In order to judge this, we need $\bm\alpha^{*(\bm w)}$, which is calculated as follows:
\begin{itemize}
\item Solve the primal problem \eqref{eq:primal} for L2-loss L1-regularized SVM by fixing $\beta^{*(\bm w)}_j = 0$ for any $j\in[d-1]$, that is,
	\begin{align*}
	& \beta^{*(\bm w)}_d = \argmin_{\beta_d} \sum_{i=1}^n w_i \ell_{y_i}(\check{x}_{id} \beta_d)
		= \argmin_{\beta_d} \sum_{i=1}^n w_i (\max\{0, 1 - y_i \beta_d\})^2 \\
	& = \argmin_{\beta_d}
		\sum_{i\in[n],~y_i=+1} w_i (\max\{0, 1 - \beta_d\})^2
		+ \sum_{i\in[n],~y_i=-1} w_i (\max\{0, 1 + \beta_d\})^2 \\
	& = \frac{\sum_{i\in[n],~y_i=+1} w_i - \sum_{i\in[n],~y_i=-1} w_i}{\sum_{i=1}^n w_i}.
	\end{align*}
\item With $\beta^{*(\bm w)}_d$ computed above and $\beta^{*(\bm w)}_j = 0$ for any $j\in[d-1]$,
	calculate
	$\bm\alpha^{\$} = \lambda\bm\alpha^{*(\bm w)} = [2\max\{0, 1 - \check{X}_{i:}\bm\beta^{*(\bm w)}\}]_{i=1}^n$
	by \eqref{eq:KKT-primal2dual-l2loss-l1reg-svm}.
\item If $| (\bm w\otimes\check{X}_{:j})^\top \bm\alpha^{\$}) | < \lambda$ for all $j\in[d-1]$,
	then $\beta^{*(\bm w)}_j = 0$ for any $j\in[d-1]$.
	So, we set $\lammax = \max_{j\in[d-1]} | (\bm w\otimes\check{X}_{:j})^\top \bm\alpha^{\$}) |$.
\end{itemize}

\subsection{All Experimental Results of Section \ref{sec:experiment-weight-changes}} \label{app:experiment}

For the experiment of Section \ref{sec:experiment-weight-changes},
ratios of screened samples by DRSsS setup is presented in Figure \ref{fig:SsS-ratio-all},
while ratios of screened features by DRSfS setup in Figure \ref{fig:SfS-ratio-all}.

\begin{figure}[t]
\begin{center}
\begin{tabular}{cc}
\begin{minipage}{0.48\hsize}
Dataset: australian
\\
\includegraphics[width=\hsize]{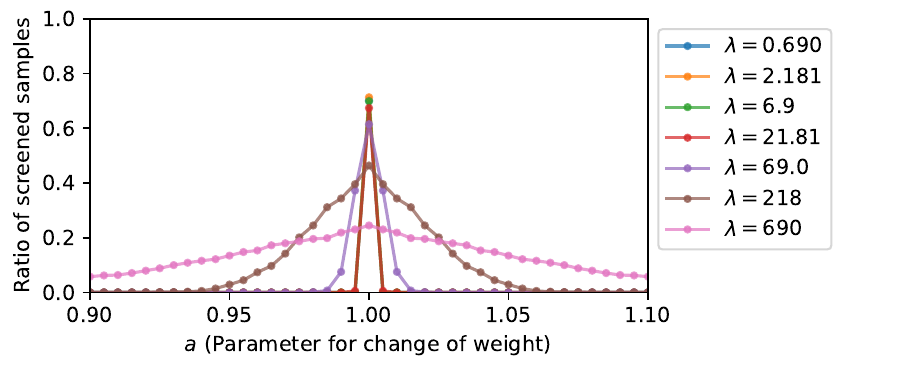}
\end{minipage}
&
\begin{minipage}{0.48\hsize}
Dataset: breast-cancer
\\
\includegraphics[width=\hsize]{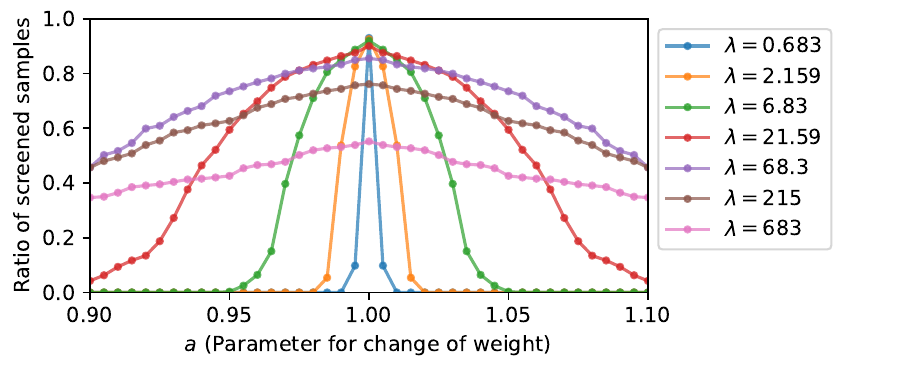}
\end{minipage}
\\
\begin{minipage}{0.48\hsize}
Dataset: heart
\\
\includegraphics[width=\hsize]{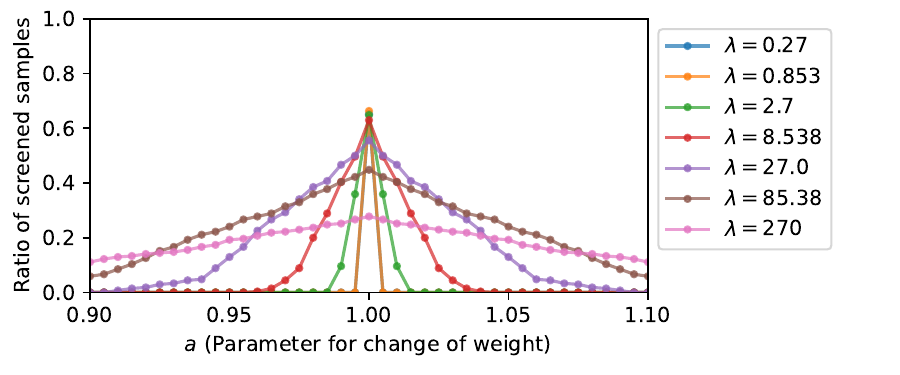}
\end{minipage}
&
\begin{minipage}{0.48\hsize}
Dataset: ionosphere
\\
\includegraphics[width=\hsize]{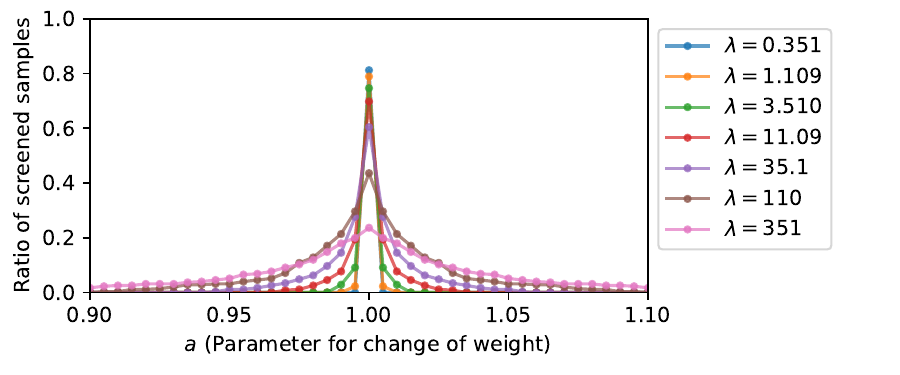}
\end{minipage}
\\
\begin{minipage}{0.48\hsize}
Dataset: sonar
\\
\includegraphics[width=\hsize]{sample-screening/svm_ss_screening_rate_Label_WeightL2Range_sonar_scale.pdf}
\end{minipage}
&
\begin{minipage}{0.48\hsize}
Dataset: splice
\\
\includegraphics[width=\hsize]{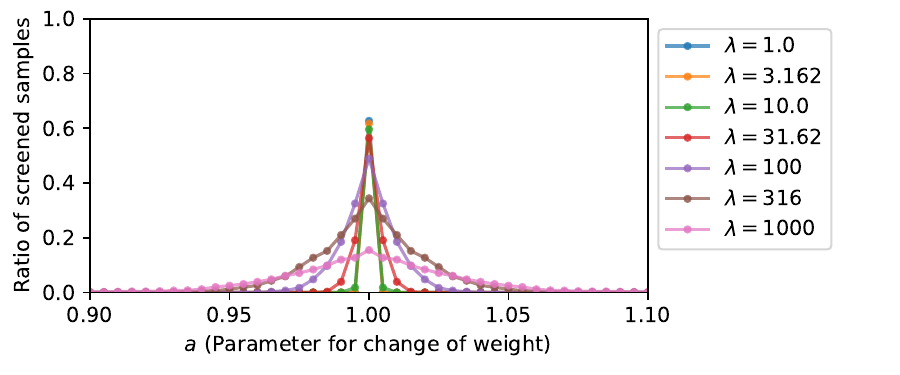}
\end{minipage}
\\
\begin{minipage}{0.48\hsize}
Dataset: svmguide1
\\
\includegraphics[width=\hsize]{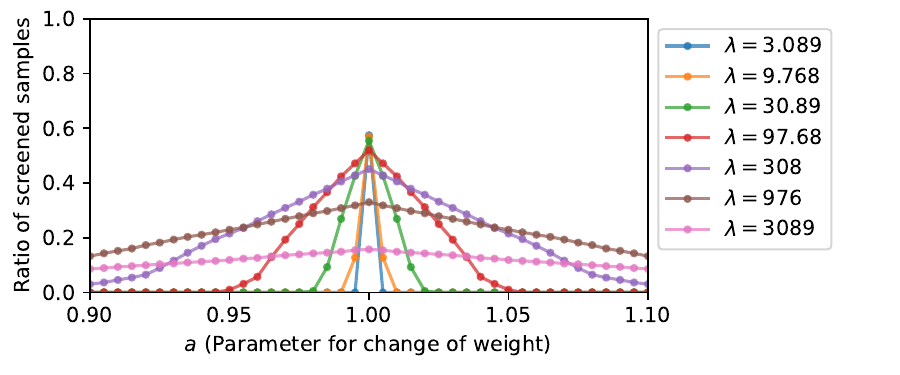}
\end{minipage}
&
~
\end{tabular}
\end{center}
\caption{Ratios of screened samples by DRSsS.}
\label{fig:SsS-ratio-all}
\end{figure}

\begin{figure}[t]
\begin{center}
\begin{tabular}{cc}
\begin{minipage}{0.48\hsize}
Dataset: madelon
\\
\includegraphics[width=\hsize]{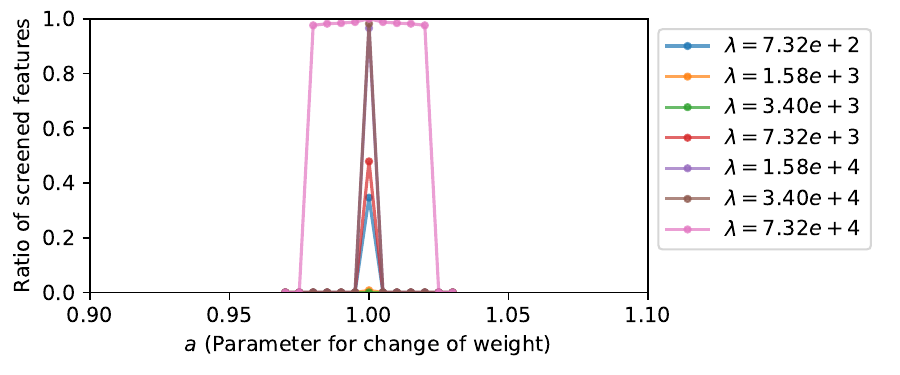}
\end{minipage}
&
\begin{minipage}{0.48\hsize}
Dataset: sonar
\\
\includegraphics[width=\hsize]{feature-screening/sfs_sonar.pdf}
\end{minipage}
\\
\begin{minipage}{0.48\hsize}
Dataset: splice
\\
\includegraphics[width=\hsize]{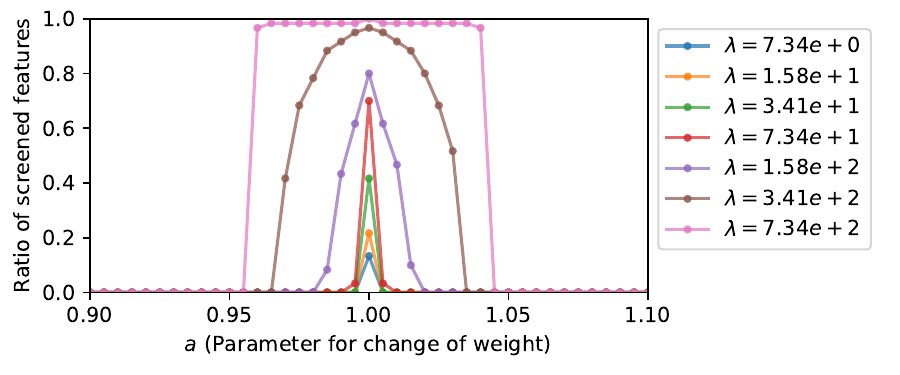}
\end{minipage}
&
~
\end{tabular}
\end{center}
\caption{Ratios of screened features by DRSfS.}
\label{fig:SfS-ratio-all}
\end{figure}


\end{document}